\pdfoutput=1
\documentclass[11pt]{article}

\usepackage[]{emnlp2021}

\usepackage{times}
\usepackage{latexsym}

\usepackage[T1]{fontenc}
\usepackage[utf8]{inputenc}

\usepackage{microtype}

\usepackage{url}            % simple URL typesetting
\usepackage{amsmath,amsthm,amsfonts,amssymb,bm,stmaryrd}       % blackboard math symbols
\usepackage{nicefrac}       % compact symbols for 1/2, etc.
\usepackage{enumitem}
\usepackage[noend]{algpseudocode}
\usepackage{algorithm}
\usepackage{mathtools}
\usepackage{nccmath}
\usepackage{multirow}
\usepackage{bigdelim}
\usepackage{color, colortbl}

\newcommand*\Let[2]{\State #1 $\gets$ #2}

\algrenewcommand\algorithmicindent{1.0em}%
\usepackage{xcolor}		% make links dark blue
\definecolor{darkblue}{rgb}{0, 0, 0.5}
\hypersetup{colorlinks=true, citecolor=darkblue, linkcolor=darkblue, urlcolor=darkblue}
\usepackage{booktabs}
\usepackage{dcolumn}
\newcolumntype{d}{D{.}{.}{-1}}
\usepackage{graphicx}
\usepackage{titlesec}
\usepackage{hyperref}
\usepackage{url}
\usepackage{framed}
\usepackage{tcolorbox}
\usepackage{microtype}
\usepackage{subcaption}
\usepackage{booktabs} % for professional tables
\usepackage{graphicx}
\usepackage[normalem]{ulem}
\useunder{\uline}{\ul}{}
\usepackage{fontawesome}
\usepackage{calc}

\definecolor{plotblue}{rgb}{0.16, 0.46, 0.69}
\definecolor{plotorange}{rgb}{0.988, 0.502, 0.1765}

\title{Consistent Accelerated Inference via Confident Adaptive Transformers}

\author{Tal Schuster$^*$ \quad Adam Fisch$^*$ \quad Tommi Jaakkola \quad Regina Barzilay \\
  Computer Science and Artificial Intelligence Laboratory\\
  Massachusetts Institute of Technology \\
  {\tt \{tals,fisch,tommi,regina\}@csail.mit.edu} \\}

\newtheorem{theorem}{Theorem}
\numberwithin{theorem}{section}

\newtheorem{lemma}[theorem]{Lemma}

\newtheorem{proposition}[theorem]{Proposition}

\theoremstyle{definition}

\newtheorem{remark}[theorem]{Remark}

\newcommand{\cset}{\mathcal{C}_{\epsilon}}

\newcommand{\newpar}[1]{\paragraph{#1.}}

\titlespacing*{\subsection}{0pt}{.35\baselineskip}{.25\baselineskip}
\titlespacing*{\subsubsection}{0pt}{.35\baselineskip}{.25\baselineskip}
\titlespacing*{\section}{0pt}{.35\baselineskip}{.25\baselineskip}
\definecolor{darkgreen}{rgb}{0.31, 0.47, 0.26}
\definecolor{Gray}{gray}{0.95}
\makeatletter
\renewcommand{\paragraph}{%
  \@startsection{paragraph}{4}%
  {\z@}{0.75ex \@plus 1ex \@minus .2ex}{-1em}%
  {\normalfont\normalsize\bfseries}%
}
\makeatother
\begin{document}
\maketitle
% TODO: uncomment
\renewcommand{\thefootnote}{\fnsymbol{footnote}}
\footnotetext[1]{The first two authors contributed equally.}
\renewcommand{\thefootnote}{\arabic{footnote}}

\begin{abstract}
We develop a novel approach for confidently accelerating inference in the large and expensive multilayer Transformers that are now ubiquitous in natural language processing (NLP). Amortized or approximate computational methods increase efficiency, but can come with unpredictable performance costs. In this work, we present CATs---\textbf{C}onfident \textbf{A}daptive \textbf{T}ransformers---in which we simultaneously increase computational efficiency, while \emph{guaranteeing} a specifiable degree of consistency with the original model with high confidence. Our method trains additional prediction heads on top of intermediate layers, and dynamically decides when to stop allocating computational effort to each input using a meta consistency classifier. To calibrate our early prediction stopping rule, we formulate  a unique extension of  conformal prediction. We demonstrate the effectiveness of this approach on four classification and regression tasks.\footnote{\faExternalLink \  \url{https://github.com/TalSchuster/CATs}}

% We present a novel dynamic Transformer for accelerated inference with computational adaptive confidence-based capacity. Recent introduced large Transformers have shown impressive performance by using many layers, but are costly and slow to run. Instead, our model can dynamically decide if to utilize the full computational power of the Transformer or to only use a subset of the layers for ``easier'' inputs. Importantly, our model can confidently decide on the computational effort to allocate for each input while \emph{preserving consistency} with the full model, up to an arbitrary tolerance level. Our solution is based on two components. First, we design a new meta predictor to assign confidence to ``early'' predictions. Then, we extend conformal prediction methods to achieve well-calibrated confidence measures that guarantee consistency. We demonstrate the effectiveness of this approach on four classification and regression tasks.
\end{abstract}

\section{Introduction}
\label{sec:introduction}

Large pre-trained language models have become the de facto standard approach for solving natural language processing tasks~\citep{devlin-etal-2019-bert, liu2019roberta}.
Despite their impressive performance, however, their often massive computational burden  makes them costly to run~\citep{schwartz2019green, sharir2020cost}.
%The recently introduced GPT-3 model~\cite{brown2020language}, for example, consists of 175B parameters across 96 Transformer layers, and requires specialized hardware to run even a single forward pass. As a consequence, these models can be challenging to use in real-world applications~\citep{schwartz2019green, sharir2020cost}. 
Concerns about their efficiency have  kindled a large body of research in the field~\cite{sanh2020distilbert, schwartz-etal-2020-right, Fan2020Reducing}. 
For multilayered architectures such as the Transformer, a popular approach is \emph{adaptive early exiting}~\cite[][\emph{inter alia}]{schwartz-etal-2020-right, xin-etal-2020-early}. Early exiting takes advantage of the observation that task instances vary in complexity. In this setting,  ``early'' classifiers are added on top of the simpler features of intermediate layers in the base model, and can trigger a prediction before the full model is executed. Naively deciding when to preempt computation, however, can result in unpredictable decreases in model accuracy.
%In response, an increasing amount of effort in the field has been focused on increasing the efficiency of large language models, using methods ranging from distillation to layer pruning~\cite[][]{sanh2020distilbert, schwartz-etal-2020-right, Fan2020Reducing}. Unfortunately, increases in efficiency inherently result in unpredictable decreases in model accuracy.

\begin{figure}[!t]
    \centering
    \includegraphics[width=.74\linewidth]{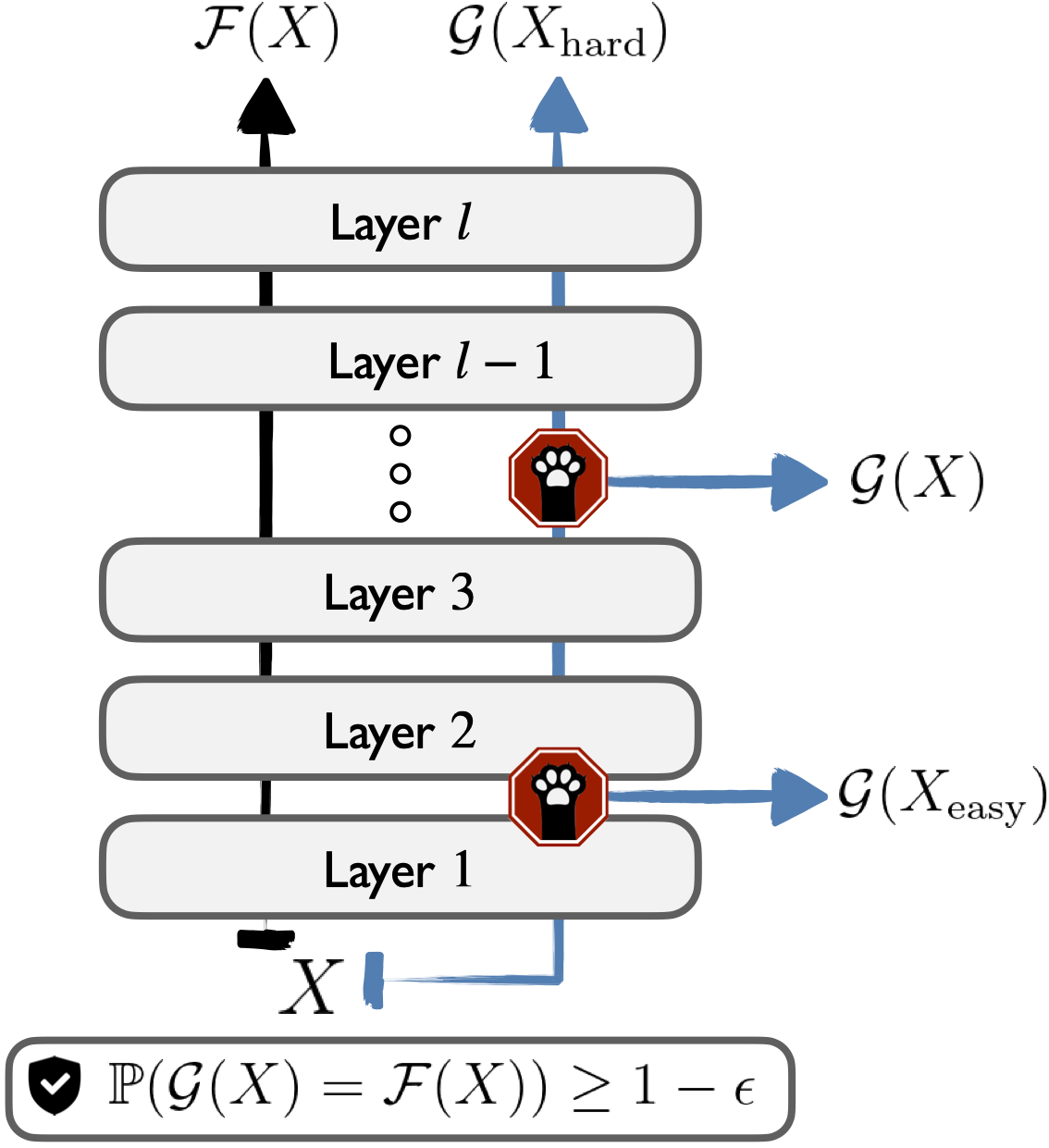}
    \caption{Our CAT model $\mathcal{G}$ can save computational resources by exiting early on certain inputs---while guaranteeing predictive consistency with the full model $\mathcal{F}$.}
    \vspace{-10pt}
    \label{fig:inference}
\end{figure}

\begin{figure*}[!t]
    \begin{tcolorbox}[colback=white, boxrule=1pt]
    \begin{center}
    \includegraphics[width=1\linewidth]{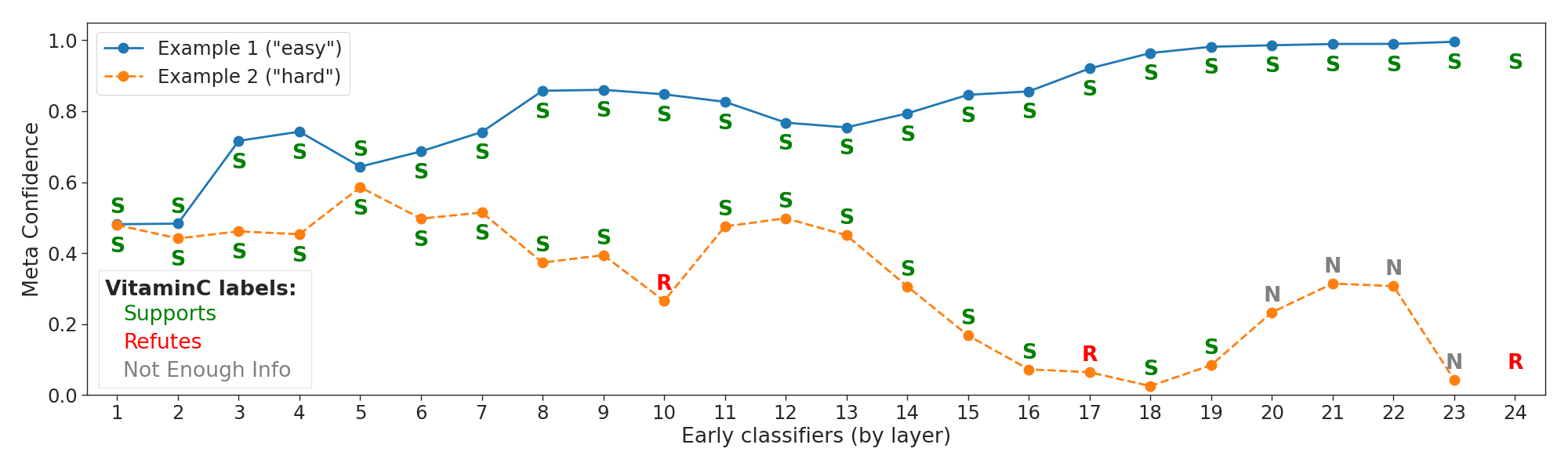} 
     \end{center}
      \vspace*{-0.4\baselineskip}
    \footnotesize
  \vspace*{-0.5\baselineskip}
        \textcolor{plotblue}{\textbf{(Ex.1)}}
        \textbf{Claim:} All airports in Guyana were closed for all international passenger flights until 1 May 2020.
        \newline 
        \phantom{\textbf{(Ex.1)} }\textbf{Evidence:}  Airports in Guyana are closed to all international passenger flights until 1 May 2020.
        \newline \vspace*{-0.5\baselineskip} \newline
        \textcolor{plotorange}{\textbf{(Ex.2)}}
        \textbf{Claim:} Deng Chao broke sales record for a romantic drama.
        \newline 
        \phantom{\textbf{(Ex.2)} }\textbf{Evidence:}  The film was a success and broke box office sales record for mainland-produced romance films.
  \vspace*{-0.6\baselineskip}
\end{tcolorbox}
    
    \small    
    \vspace{-5pt}
    \caption{Confidence levels given by our meta model regarding the \emph{consistency} of our prediction as computation progresses. 
    %Our meta model $\mathcal{M}_k$ estimates the likelihood that the prediction of a preemptive model after the computation of intermediate layer $k$ matches the prediction of the full model (after $l > k$ total layers).
    Ex.1 from the VitaminC fact verification dataset is ``easy'', and is classified consistently by all early classifiers $\mathcal{F}_k$ ($\texttt{Supports}$). The meta confidence captures this, and increases with time. Ex.2 is harder---and the prediction changes ($\texttt{Refutes} / \texttt{NEI}$) as it propagates though the Transformer layers. Appropriately, the meta confidence is  low. The exact exit layer of $\mathcal{G}$ is determined as a function of a user-specified tolerance $\epsilon$, see Eq.~\eqref{eq:marginalcoverage}.} %($\mathcal{R}_{\epsilon}$).}
    \vspace{-12pt}
    \label{fig:intro}
\end{figure*}
Quantifying the \emph{uncertainty} in a prediction in order to decide when additional computation is needed (or not) is critical to making predictions quickly without excessively sacrificing performance.
%  (or  not) Make predictions quickly while maintaining  performance is a trade-off that is critical to production-level machine learning systems. 
In this paper, we present \textbf{C}onfident \textbf{A}daptive \textbf{T}ransformers (CATs), a general method for increasing Transformer-based model efficiency while remaining \emph{confident} in the quality of our predictions. Specifically, given a fixed, expensive $l$-layer model $\mathcal{F}(x)$, we create an amortized model $\mathcal{G}(x)$ that includes early classifiers $\{\mathcal{F}_1, \ldots, \mathcal{F}_l\}$.\footnote{We simply define the final  $\mathcal{F}_l$ as $\mathcal{F}_l(x)  \triangleq \mathcal{F}(x) ~~ \forall x$.} We then make $\mathcal{G}$ provably consistent with the original $\mathcal{F}$ with arbitrarily high probability (e.g., 95\% of the time). This process is illustrated in Figure~\ref{fig:inference}.

%We denote an early classifier acting at layer $k$ out of $l$ total as $\mathcal{F}_k$, and the aggregated adaptive Transformer as $\mathcal{G} = \{\mathcal{F}_1, \ldots, \mathcal{F}_l\}$.\footnote{$\mathcal{F}_l(x)=\mathcal{F}(x) ~~ \forall x$, since we keep $\mathcal{F}$ fixed.}

Our approach builds on conformal prediction (CP), a model-agnostic and distribution-free framework for creating well-calibrated predictions~\cite{vovk2005algorithmic}.
Concretely, suppose we have been given $n$ examples, $X_i \in \mathcal{X}$, $i=1,\ldots,n$, as \linebreak \emph{unlabeled} calibration data, that have been drawn exchangeably from some underlying distribution $P$. Let $X_{n+1} \in \mathcal{X}$ be a new exchangeable test example for which we would like to make a prediction. 
The aim of our method is to construct $\mathcal{G}$ such that it agrees with $\mathcal{F}$ with \emph{distribution-free marginal coverage} at a tolerance level $\epsilon \in (0, 1)$, i.e.,
%
% \addtolength\abovedisplayskip{0ex}
% \addtolength\belowdisplayskip{0ex}
\begin{equation}
    \label{eq:marginalcoverage}
    \mathbb{P}\big(\mathcal{G}(X_{n+1}) = \mathcal{F}(X_{n+1})\big) \geq 1 - \epsilon.
\end{equation}
We consider $\mathcal{G}$ to be \emph{$\epsilon$-consistent} if the frequency of error, $\mathcal{G}(X_{n+1}) \neq \mathcal{F}(X_{n+1})$, does not exceed $\epsilon$.\footnote{For regression, we define  equality  as $|\mathcal{G}(\cdot) - \mathcal{F}(\cdot)| \leq \tau$.}
By design, this ensures that $\mathcal{G}$ preserves at least $(1-\epsilon)$-fraction of $\mathcal{F}$'s original performance. Within these constraints, the remaining challenge is to make $\mathcal{G}$ relatively efficient (e.g., a consistent, but vacuous, model is simply the identity $\mathcal{G} \triangleq \mathcal{F}$).

In order to support an efficient $\mathcal{G}$, we need a reliable signal for inferring whether or not the current prediction is likely to be stable. Past work~\cite[e.g.,][]{schwartz-etal-2020-right} rely on potentially poorly correlated metrics such as the early classifier's {softmax response}. We address this challenge by instead directly learning meta ``consistency predictors'' for each of the $l - 1$ early classifiers of our $l$ layer model, by leveraging patterns in past predictions.\footnote{We refer to the meta aspect of the classifier, not the optimization process (i.e., not to be confused with \emph{meta-learning}).}  Figure~\ref{fig:intro} demonstrates the progression of meta confidence scores across layers when applied to ``easy'' versus ``hard'' instances from the VitaminC fact verification task~\cite{Schuster2021}.

We  pair the scores of our meta classifier for each layer with a stopping rule that is calibrated using a unique twist on standard conformal prediction. Traditionally, CP is used to construct  prediction \emph{sets} that cover the desired target (e.g., $Y_{n+1}$) with high probability. We invert the CP problem to first infer the multi-label set of \emph{inconsistent} layers, and then exit at the first layer that falls in its complement. We then demonstrate that this can be reduced to setting a simple (but well-calibrated) exit threshold for the meta classifier scores. Our resulting algorithm is (1) fast to compute in parallel to the main Transformer, (2) requires only unlabeled data, and (3) is statistically efficient in practice, in the sense that it finds low exit layers on average while still maintaining the required predictive consistency.
%To form an efficient $\mathcal{G}$, we follow recent work on \emph{adaptive early exiting}~\cite[as in][\emph{inter alia}]{xin-etal-2020-deebert, schwartz-etal-2020-right}, where we take advantage of the observation that task instances vary in complexity, and train ``early'' classifiers on top of the  simpler features computed by intermediate layers in the Transformer model. We denote an early classifier acting at layer $k$ out of $l$ total as $\mathcal{F}_k$, and the aggregated adaptive Transformer as $\mathcal{G} = \{\mathcal{F}_1, \ldots, \mathcal{F}_l\}$.\footnote{$\mathcal{F}_l(x)=\mathcal{F}(x) ~~ \forall x$, since we keep $\mathcal{F}$ fixed.} After the evaluation of each $\mathcal{F}_k \in \mathcal{G}$, we choose whether or not to keep the prediction (and stop further computation). This process is illustrated in Figure~\ref{fig:inference}. Our method applies to any multilayer network, without requiring any retraining or distillation of the original parameters.

% by \emph{meta-learning} a calibrated stopping rule $\mathcal{M}_k(X_{n+1}, \mathcal{G})$
%Naively deciding when to preempt computation, however, will not lead to $\epsilon$-{consistent} models in general. 

We validate our method on four diverse NLP tasks---covering both classification and regression, different label space sizes, and varying amounts of training data.
We find that it constitutes a simple-yet-effective approach to confident adaptive prediction with minimal interventions and desirable theoretical guarantees.
In short, we provide:
\vspace{-7.5pt}
\begin{enumerate}[leftmargin=*, noitemsep]
    \item A novel theoretical extension of conformal prediction to accommodate adaptive prediction;\vspace{5pt}
    \item An effective meta consistency classifier for deriving a confident ``early exiting'' model;\vspace{5pt}
    \item A demonstration of the utility of our framework on both classification and regression tasks, where we show significant efficiency improvements, while {guaranteeing} high consistency.
    %on the range of tasksConsistent and significant empirical gains across four diverse domains and multiple target performance levels.
\end{enumerate}

\section{Related Work}
\newpar{Adaptive computation}
Reducing the computational cost of neural models has received intense interest. % from the machine learning community. 
% %Consequently, many efficiency improvements have been proposed in recent years.
% A popular approach is model compression---either by distillation~\cite{sanh2020distilbert}, in which a smaller model is trained to imitate the larger one, or by model pruning~\cite{Fan2020Reducing,michel2019heads}, in which parameters are removed from the original model. Both approaches result in a single model that is used for all future inputs.
Adaptive approaches adjust the amount of computation per example to \emph{amortize} the total inference cost~\cite[see][\emph{inter alia}]{teerapittayanon2017branchynet, graves2017adaptive, huang2018multiscale, pmlr-v97-kaya19a, Wang_2018_ECCV}. As discussed in \S\ref{sec:introduction}, our method is inspired by the approach of \citet{schwartz-etal-2020-right} and others \cite{liu-etal-2020-fastbert, geng2021romebert, bert_loses_patience}, where they preempt computation if the softmax value of any early classifier is above a predefined threshold. Yet unlike our approach,  their model is not guaranteed to be accurate. %, even after softmax calibration~\cite{guo-2017-deep}.
%Several approaches to early exiting also include fine-tuning stages to improve efficiency~\cite{liu-etal-2020-fastbert, geng2021romebert, bert_loses_patience}. In this work, we choose to avoid this step in order to allow our model to be widely applied with minimal overhead.
%
In concurrent work, \citet{xin-etal-2021-berxit} propose a meta confidence classifier similar to ours. However, as in previous work, they do not address the calibration part to guarantee consistency.
% As in previous work, however, they do not address how to calibrate their model to guarantee a particular desired level of performance.

% \newpar{Uncertainty estimation} A large amount of research in model confidence has been dedicated towards \emph{calibrating} the model posterior,  $p_\theta(\hat{y}_{n+1}|x_{n+1})$, such that the true accuracy, $y_{n+1} = \hat{y}_{n+1}$, 
% %where $\hat{y}_{n+1} = \arg\!\max p_\theta(y_{n+1}|x_{n+1})$,
% is indeed equal to the estimated probability~\cite[][]{niculescu2005predicting, lakshinarayanan2017ensemble, lee2018training}. In theory, these estimates could be used to create confident early exit conditions (such as when the posterior exceeds a certain threshold)---and this is indeed the approach taken by \citet{schwartz-etal-2020-right}. However, these methods are not guaranteed to be accurate with finite amounts of training data, and often still suffer from miscalibration in practice. This is especially true for modern neural networks~\cite{pmlr-v70-guo17a, ashukha2020pitfalls,  hirschfeld2020uncertainty}. 
% Alternatives such as Bayesian approaches can be useful and satisfy different desiderata~\cite[][]{neal1996bayesian, graves2011vi, hernandez2015bnn, gal2016dropout}, but do not necessarily give well-calibrated probabilities (e.g., due to model misspecification, choice of prior, or approximation issues). Our CP-based method, on the other hand, gives well-calibrated distribution-free guarantees---even in finite samples.  

\newpar{Confident prediction} A large amount of research has been dedicated towards \emph{calibrating} the model posterior,  $p_\theta(\hat{y}_{n+1}|x_{n+1})$, such that the accuracy, $y_{n+1} = \hat{y}_{n+1}$, 
% %where $\hat{y}_{n+1} = \arg\!\max p_\theta(y_{n+1}|x_{n+1})$,
is indeed equal to the estimated probability~\cite[][]{niculescu2005predicting, gal2016dropout, pmlr-v70-guo17a}. In theory, these estimates could be leveraged to create confident early exits---e.g., similar to  \citet{schwartz-etal-2020-right}. Ensuring calibrated probabilities of this form is hard, however, and existing methods often still suffer from miscalibration. Additionally, many methods exist for bounding the true error of a classifier~\cite{langford2005, park2021pac}, but do not give end-users opportunities to control it. More similar to our work, selective classification~\cite{geifman2017selective} allows the model to abstain from answering when not confident, in order to maintain a target error rate only over answered inputs. Our work gives a different and statistically efficient technique applied to \emph{consistent} prediction.

\newpar{Conformal prediction} CP~\cite{vovk2005algorithmic} typically is formulated in terms of prediction \emph{sets} $\mathcal{C}(X_{n+1})$, where finite-sample, distribution-free guarantees can be given over the event that $\mathcal{C}$ contains $Y_{n+1}$. As we discuss in \S\ref{sec:conformal}, internally our method follows a similar approach in which we try to conservatively identify the inadmissible set of all layers that are \emph{inconsistent} (and exit at the first layer that falls in that set's complement). Most relevant to our work, \citet{cauchois2020knowing} presents algorithms for conformal multi-label predictions. We leverage similar methods in our model, but formulate our solution in terms of the \emph{complement} of a multi-label set of inconsistent \emph{predictions}. Our work adds to several recent directions that explore CP in the context of risk-mitigating applications~\cite[][\emph{inter alia}]{lei2020causal, Romano2020With, bates-rcps, fisch2021admission}, or meta-learning settings~\cite{fisch2021meta}.

\section{Early Exiting Transformers}\label{sec:early_trans}
In the following, we describe our dynamic early exiting model. We summarize early classification (following previous work) for convenience (\S\ref{sec:preemptive}), and then present our novel meta consistency classifier (\S\ref{sec:meta}). We focus on classification and regression tasks, given a model $\mathcal{F}(x) = y$. We assume that $\mathcal{F}$ maps the input $x \in \mathcal{X}$ into a series of feature representations before making the prediction $y \in \mathcal{Y}$. Here, $\mathcal{F}$ is a multilayered Transformer~\cite{transformer} composed of $l$ layers (although our method can be applied to any multilayer network).

For all downstream tasks we follow standard practice and assume that the input contains a $\texttt{[CLS]}$ token whose representation is used for prediction.
%is either a sentence $x_1$, or a pair of sentences $(x_1, x_2)$. %Following standard practice, we take the model input $x$ as $x = \texttt{[CLS]} x_1 \texttt{[SEP]}$,  or $x = \texttt{[CLS]} x_1 \texttt{[SEP]} x_2 \texttt{[SEP]}$.
For classification, we use a task-specific head, $\mathrm{softmax}(\mathbf{W}_o(\phi(\mathbf{W}_p\mathbf{h}_{\texttt{[CLS]}})))$, where $\mathbf{h}_{\texttt{[CLS]}} \in \mathbb{R}^d$ is the hidden representation of the $\texttt{[CLS]}$ token,\footnote{$d$ varies by $\mathcal{F}$. In Albert-xlarge $d=2048$.} $\phi$ is a nonlinear activation, and $\mathbf{W}_*$ are linear projections, where $\mathbf{W}_p \in \mathbb{R}^{d \times d}$ and $\mathbf{W}_o \in \mathbb{R}^{|\mathcal{Y}| \times d}$. % Following the common practice for sequence classification Transformers~\citep{devlin-etal-2019-bert}, for the top layer we use $\mathbf{W}_p^{(l)} \in \mathbb{R}^{d \times d}$ and $\mathbf{W}_o^{(l)} \in \mathbb{R}^{|\mathcal{Y}| \times d}$. %For all earlier classifiers, we reduce the first projection to $d_e < d$, resulting in $\mathbf{W}_p^{(k)} \in \mathbb{R}^{d_e \times d}$ and $\mathbf{W}_o^{(k)} \in \mathbb{R}^{|\mathcal{Y}| \times d_e}$.
Regression is treated similarly, but uses a 1-d output projection, $\mathbf{w}_o \cdot \mathbf{h}_{\texttt{[CLS]}}$.

\subsection{Early predictors}\label{sec:preemptive}
$\mathcal{F}$'s structure yields a sequence of hidden $\texttt{[CLS]}$ representations, $\{\mathbf{h}^{(1)}_\texttt{[CLS]}, \ldots, \mathbf{h}^{(l)}_\texttt{[CLS]}\}$, where $\mathbf{h}^{(k)}_\texttt{[CLS]} \in \mathbb{R}^d$ is the representation after applying layer $k$.
%Even without any special fine-tuning, we observe that $\mathbf{h}^{(k)}_\texttt{[CLS]}$ naturally becomes more informative with respect to the output $y$ as $k \rightarrow l$. For easy $x$, many intermediate $\mathbf{h}^{(k)}_\texttt{[CLS]}$ ($k \ll l$), are informative enough to make an accurate prediction. 
%
After each intermediate layer $k < l$, we train an early classification head that is similar to the head used in $\mathcal{F}$, but reduce the dimensionality of the first projection to $\mathbf{W}_p^{(k)} \in \mathbb{R}^{d_e \times d}$ (this is purely for efficiency\footnote{We simply set $d_e=32$ in all our experiments.}). The final $\mathcal{F}_l$ is unchanged from $\mathcal{F}$. These  extra $(l - 1) \times (d_e \times d + d_e \times |\mathcal{Y}|)$ parameters are quick to tune on top of a fixed $\mathcal{F}$, and we can reuse $\mathcal{F}$'s training data as $\mathcal{D}_{\mathrm{tune}}$.\footnote{Or if $\mathcal{D}_{\mathrm{tune}}$ is unlabeled, we can use $\mathcal{F}(x)$ as labels.} The classifier $\mathcal{F}_k(x) = \mathrm{softmax}(\mathbf{W}^{(k)}_o(\phi(\mathbf{W}^{(k)}_p\mathbf{h}^{(k)}_{\texttt{[CLS]}})))$ is then used after layer $k$ to get an early prediction candidate. Early regression is handled similarly.

\subsection{Meta early exit classifier}\label{sec:meta}

To decide \emph{when} to accept the current prediction and stop computation, we require some signal as to how likely it is that $\mathcal{F}_k(x) = \mathcal{F}(x)$. Previous work relies on intrinsic measures (e.g., softmax response). Here, we present a \emph{meta} classifier to explicitly estimate the consistency of an early predictor.
Given fixed $\mathcal{F}_k$ and $\mathcal{F}$, we train a small binary MLP, $\mathcal{M}_k(x) \in \mathbb{R}$, on another \linebreak \emph{unlabeled} (limited) sample of task in-domain data, $\mathcal{D}_{\mathrm{meta}}$.
% \footnote{We can also reuse a sample of $\mathcal{F}$'s training data distinct of $\mathcal{D}_{\mathrm{tune}}$.}
%
% During training, we minimize a binary cross entropy loss:
% \begin{equation}
%     \mathcal{L}_{BCE} = -\sum_{x} \mathbf{1}\{\mathcal{F}_k(x) = \mathcal{F}(x)\}\log \mathcal{M}_k(x) + \mathbf{1}\{\mathcal{F}_k(x) \neq \mathcal{F}(x)\}(1 - \log \mathcal{M}_k(x))
% \end{equation}
As input, % to $\mathcal{M}_k$ for layer $k$,
we provide the current ``early'' hidden state $\phi(\mathbf{W}^{(k)}_p\mathbf{h}^{(k)}_{\texttt{[CLS]}})$, in addition to several processed meta features, see Table~\ref{tab:features}. We then train $\mathcal{M}_k$ with a binary cross entropy objective, where we maximize the likelihood of predicting $\mathbf{1}\{\mathcal{F}_k(x_i) = \mathcal{F}(x_i)\}$ for $x_i \in \mathcal{D}_{\mathrm{meta}}$.%\footnote{Model parameters are shared across $\mathcal{M}_{1:l}$.}

Using the trained $\mathcal{F}_k$ and $\mathcal{M}_k$, we define the full adaptive model $\mathcal{G}$ using the prediction rule
\begin{equation}
\label{eq:g}
\resizebox{.89\hsize}{!}{$\displaystyle
     \mathcal{G}(x; \bm{\tau}) := \left\{
     \begin{array}{@{}l@{\thinspace}l}
       \mathcal{F}_1(x)  &\hspace{6mm} \text{if } \mathcal{M}_1(x) > \tau_1, \\
       \mathcal{F}_2(x)  &\hspace{6mm} \text{else if } \mathcal{M}_2(x) > \tau_2, \\
       &\hspace{2mm}\vdots \\ 
       \mathcal{F}_l(x)  &\hspace{6mm} \text{otherwise}, \\
     \end{array}
   \right.
 $}
\end{equation}
where $\bm{\tau} = (\tau_1, \ldots, \tau_{l-1})$ are confidence thresholds. The key challenge is to \emph{calibrate} $\tau_k$ such that $\mathcal{G}$ guarantees $\epsilon$-consistent performance per Eq.~\eqref{eq:marginalcoverage}.

\begin{table}[t]
\centering
\small
\begin{tabular}{p{0.14\linewidth} p{0.70\linewidth}}
\toprule
Meta Feature        & Description                       \\ \midrule
$\hat{y}_k$         & The current prediction. \\
$\mathrm{history}$    & The past $k - 1$ predictions, $\hat{y}_{1:k-1}$\newline (For classification we give $p_k(\hat{y}_k | x)$).                             \\
% $\mathrm{progress}$ & The current layer progress, $k / l$.      \\[2mm]
%\multicolumn{2}{c}{{\ul \textit{Classification only}}}  \\[2mm]
$p_k^{\mathrm{max}}$  & Prob.\ of the prediction, $p_k(\hat{y}_k | x)$.                                                                 \\
$p_k^{\mathrm{diff}}$ & Difference in prob.\ of top predictions, $p_k(\hat{y}_k | x) - {\arg\!\max}_{y_k \neq \hat{y}_k}~p_k(y_k | x)$. \\ \bottomrule
\end{tabular}
\vspace{-5pt}
\caption{Additional meta features used as input to the meta early exit classifier, $\mathcal{M}_k$. Where specified, the probability $p_k$ is taken from the model's early softmax. $p_k^{\mathrm{max}}$ and $p_k^{\mathrm{diff}}$ are only used for classification tasks.} %  All features are simply concatenated together.}
\vspace{-10pt}
\label{tab:features}
\end{table}

\subsection{Warmup: development set calibration}
\label{sec:naive}
A simple approach to setting $\bm{\tau}$ is to optimize performance on a development set $\mathcal{D}_{\mathrm{dev}}$, subject to a constraint on the empirical inconsistency:
\begin{equation}
\begin{split}
 &\bm{\tau}^* :=~\underset{(\tau_1, \ldots, \tau_{l-1})}{\mathrm{minimize}}~~ \widehat{\mathbb{E}}_{\mathrm{dev}}[\mathrm{exit}(G(X; \bm{\tau}))] \\
 &\hspace{0cm}\text{s.t.  }~\widehat{\mathbb{E}}_{\mathrm{dev}}[\mathbf{1}\{\mathcal{G}(X; \bm{\tau}) = \mathcal{F}(X)\}] \ge 1-\epsilon,
\end{split}\label{eq:dev_set_thres}
\end{equation}
% \begin{equation}
% \begin{split}
%  \bm{\tau}^*& :=~\underset{(\tau_1, \ldots, \tau_{l-1})}{\mathrm{minimize}}~~ \widehat{\mathbb{E}}_{\mathrm{dev}}[|\mathcal{G}(X; \bm{\tau})|] \\
%  &\hspace{0cm}\text{s.t.  }~\widehat{\mathbb{E}}_{\mathrm{dev}}[\mathbf{1}\{\mathcal{G}(X) = \mathcal{F}(X)\}] \ge 1-\epsilon,
% \end{split}\label{eq:dev_set_thres}
% \end{equation}
where $\mathrm{exit}(\cdot)$ measures the exit layer, and  $\widehat{\mathbb{E}}_{\mathrm{dev}}$ is simply the average over $\mathcal{D}_{\mathrm{dev}}$. Using a standard error bound~\cite[][]{langford2005} over a separate split, $\mathcal{D}_{\mathrm{cal}}$, we can then derive the following guarantee:

\begin{proposition}
\label{prop:naive}
Let $X_i$, $i = 1, \ldots, n$ be an i.i.d. sample with $s = \sum_{i=1}^n \mathbf{1}\{\mathcal{G}(X_i; \bm{\tau}) = \mathcal{F}(X_i)\}$. Then, up to a confidence level $\delta$, we have that
\begin{equation}
\label{eq:test_bound}
    \mathbb{P}(\mathbb{P}(\mathcal{G}(X; \bm{\tau}) = \mathcal{F}(X)) \geq 1 - \tilde{\epsilon}) \geq 1 - \delta,
\end{equation}
where $\tilde{\epsilon}$ is the solution to $\mathrm{Beta}(s, n - s + 1) = \delta$, and $\mathrm{Beta}$ is the incomplete beta function.
\end{proposition}
A proof is given in Appendix~\ref{app:proofs}.
Though in practice $\tilde{\epsilon}$ might be close to $\epsilon$ for most well-behaved distributions, unfortunately Eq.~\eqref{eq:test_bound} does not give a fully \emph{specifiable} guarantee as per Eq.~\eqref{eq:marginalcoverage}. Readjusting $\bm{\tau}$ based on $\mathcal{D}_{\mathrm{cal}}$ requires correcting for multiple testing in order to remain theoretically valid, which can quickly become statistically inefficient.
In the next section, we provide a novel calibration approach that allows us to guarantee a target performance level with strong statistical efficiency.

\section{Conformalized Early Exits}
\label{sec:conformal}
We now formulate the main contribution of this paper, which is a \emph{distribution-free} and \emph{model-agnostic} method based on CP for guaranteeing any performance bound an end-user chooses to specify.\footnote{See \citet{shafer2008tutorial} for a concise review of CP.} 
Our training (\S\ref{sec:early_trans}), conformal calibration (\S\ref{sec:conformal}), and inference pipelines are summarized in Algorithm~\ref{alg:overall}.

\subsection{Conformal formulation}
Let $\mathcal{I}(x) := \{i \colon \mathcal{F}_{i}(x) \neq \mathcal{F}(x)\}$ be the index set of layers that are \emph{inconsistent} with the final model's prediction. To maintain $\epsilon$-consistency, we must avoid using any of the predictions specified by this set, $\mathcal{F}_i(x)$ where $i \in \mathcal{I}(x)$, more than $\epsilon$-fraction of the time for $x \in \mathcal{X}$.
In \S\ref{sec:calibration}, we show how $\mathcal{M}_{1:l-1}$ can be paired with a conformal procedure to obtain calibrated thresholds $\bm{\tau} = (\tau_1, \ldots, \tau_{l-1})$ such that we obtain a conservative prediction of $\mathcal{I}(x)$,
\begin{equation}
    \label{eq:c}
    \cset(x) := \{k \colon \mathcal{M}_k(x) \leq \tau_k\},
\end{equation}
where we ensure that $\mathcal{I}(x) \subseteq \cset(x)$ with probability at least $1 - \epsilon$. 
Proposition~\ref{prop:early_exit} states our guarantee when $\bm{\tau}$ is paired with $\mathcal{G}$ following Eq.~\eqref{eq:g}.
\begin{proposition}%[Conformal early exits]
\label{prop:early_exit}
Assume that examples $X_i$, $i=1,\ldots, n+1$ are exchangeable. For any $\epsilon \in (0, 1)$, let the index set $\cset$ (based on the first $n$ examples) be the output of conformal procedure satisfying
\begin{equation}
\label{eq:subset}
\mathbb{P}(\mathcal{I}(X_{n+1}) \subseteq \cset(X_{n+1})) \geq 1 - \epsilon.
\end{equation}
Define $K := \min \{j: j \in \cset^c(X_{n+1})\}$, the first exit layer selected by $\mathcal{G}$ following Eq.~\eqref{eq:g}.\footnote{Here $A^c$ denotes the complement index set $\{i \colon i \not\in A\}$.} Then
\begin{equation}
\mathbb{P}(\mathcal{F}_K(X_{n+1}) = \mathcal{F}(X_{n+1})) \geq 1 - \epsilon.
\end{equation}
\end{proposition}
\begin{remark}
Note that Eq.~\eqref{eq:subset} is stricter than necessary. Fundamentally, we only require that $\mathbb{P}(K \in \mathcal{I}^c(X_{n+1})) \geq 1 - \epsilon$. Nevertheless, Eq.~\eqref{eq:subset} is easier to calibrate, and leads to strong empirical results despite being theoretically conservative. 
\end{remark}
\begin{remark}
During inference we do not fully construct $\cset$; it is only used to calibrate $\bm{\tau}$ beforehand.
\end{remark}
\subsection{Conformal calibration}
\label{sec:calibration}
We now describe our conformal procedures for calibrating $\bm{\tau}$.
Conformal prediction is based on hypothesis testing, where for a given input $x$ and possible output $y$, a statistical test is performed to accept or reject the null hypothesis that the pairing $(x, y)$ is correct. In our setting, we consider the null hypothesis that layer $k$ is \emph{inconsistent}, and we use $\mathcal{M}_k(x)$ as our test statistic. Since $\mathcal{M}_k$ is trained to predict  $\mathbf{1}\{\mathcal{F}_k(x_i) = \mathcal{F}(x_i)\}$, a high value of $\mathcal{M}_k(x)$ indicates how ``surprised'' we would be if layer $k$ was in fact inconsistent with layer $l$ for input $x$. Informally, a low level of surprise indicates that the current input ``conforms'' to past data.
To rigorously quantify the degree of conformity via the threshold $\tau_k$ for predictor $\mathcal{M}_k$, we use a held-out set of $n$ unlabeled, exchangeable examples, $\mathcal{D}_{\mathrm{cal}}$.% These examples do not need to be labeled. In fact, they can even be a batch of test examples in many problems. %\footnote{Again, the examples in $\mathcal{D}_{\mathrm{cal}}$ do not need to be labeled.}
% the degree of nonconformity It is important that $\mathcal{S}$ preserves exchangeability; i.e., $\mathcal{S}$ should be symmetric with respect to permutations of its inputs. In this work we use the ``split'' CP approach~\cite{Papadopoulos08}, where we treat $\mathcal{D}$ as a proper training set, independent of any of the subsequent $n+1$ exchangeable calibration examples used for CP. Specifically, as described in \S\ref{sec:meta}, we use $\mathcal{D}_{\mathrm{meta}}$ to train $\mathcal{M}_k$, and will now calibrate $\mathcal{M}_k$ on another held-out split of data, $\mathcal{D}_{\mathrm{cal}}$.\footnote{Again, the examples in $\mathcal{D}_{\mathrm{cal}}$ do not need to be \emph{labeled}.}

\subsubsection{Independent calibration} As a first approach, we construct $\cset(x)$ by composing $l -1$ separate tests for $\mathcal{F}_k(x) \neq \mathcal{F}(x)$, each with significance $\alpha_k$, where $\alpha_k$ are corrected for multiple testing.
%=  \omega_k\cdot \epsilon$ where $\sum_{k = 1}^{l} \omega_k = 1$ (i.e., a weighted Bonferroni correction). Our default is to weight each test equally, setting $\omega_k = l^{-1}$.
%
Let $v_k^{(1:n, \infty)}$ denote the inflated empirical distribution of inconsistent layer scores, $$\{\mathcal{M}_k(x_i) \colon x_i \in \mathcal{D}_\mathrm{cal},  \mathcal{F}_k(x_i) \ne \mathcal{F}(x_i) \} \cup \{\infty\}.$$ Inflating the empirical distribution is critical to our finite sample guarantee, see Appendix~\ref{app:proofs}. We then define $\tau^{\mathrm{ind}}_k = \mathrm{Quantile}\big(1 - \alpha_k, v_k^{(1:n, \infty)}\big)$, and predict the inconsistent index set  at $x \in \mathcal{X}$ as
\begin{equation}
\mathcal{C}_\epsilon^{\mathrm{ind}}(x) = \left\{k \colon \mathcal{M}_k(x) \leq \tau_k^{\mathrm{ind}}\right\}.
% \resizebox{.89\hsize}{!}{$\displaystyle
%  \hspace{-2mm}\mathcal{C}_\epsilon^{\mathrm{ind}}(x) = \left\{k \colon
%     \mathcal{M}_k(x) \leq \mathrm{Q}\big(1 - \tilde{\epsilon}_k, v_{k}^{(1:n, \infty)}\big)\right\}.\hspace{-1pt}
% $}
%\vspace{7pt}
\end{equation}
The following theorem states how to set each  $\alpha_k$ such that the quantiles $\tau_k^{\mathrm{ind}}$ yield a valid $\cset^{\mathrm{ind}}$.
\begin{theorem}
\label{thm:ind}
Let $\alpha_k = \omega_k \cdot \epsilon$, where $\omega_k$ is a weighted Bonferroni correction, i.e., $\sum_{k = 1}^{l -1} \omega_k = 1$. Then $\cset^{\mathrm{ind}}(X_{n+1})$ is a valid set that satisfies Eq.~\eqref{eq:subset}.
\end{theorem}
\begin{remark}
$\omega_{1:l-1}$ can be tuned on a development set $\mathcal{D}_{\mathrm{dev}}$ as long as $\mathcal{D}_{\mathrm{dev}}$ is distinct from $\mathcal{D}_{\mathrm{cal}}$.
\end{remark}

\begin{figure}[t!]
    \centering
    \hypersetup{hidelinks}
    
     \vspace{-8pt}
     
    \begin{minipage}{1\linewidth}
    \begin{algorithm}[H]
    \footnotesize
    \caption{ \small Consistent accelerated inference.}
    \label{alg:overall}
    \textbf{Definitions:} $\mathcal{F}$ is a multilayered classifier trained on $\mathcal{D}_{\mathrm{train}}$. $\mathcal{D}_{\mathrm{tune}}$, $\mathcal{D}_{\mathrm{meta}}$ and $\mathcal{D}_{\mathrm{scale}}$ are collections of in-domain unlabeled data points (in practice, we reuse $\mathcal{D}_{\mathrm{train}}$ and divide it to 70/20/10\%, respectively). $\mathcal{D}_{\mathrm{cal}}$ has in-domain unlabeled examples not in $\mathcal{D}_{\mathrm{train}}$ (In practice, we take a subset of the task's validation set). $\epsilon$ is the user-specified consistency tolerance.
    \vspace{4pt}
    \begin{algorithmic}[1]

        \Function{Train }{$\mathcal{F}$, $\mathcal{D}_{\mathrm{tune}}$, $\mathcal{D}_{\mathrm{meta}}$}
        
            \State \textcolor{darkgreen}{\emph{\# Learns $\mathcal{F}_{{1}\dots{l-1}}$ and $\mathcal{M}_{{1}\dots{l-1}}$} components} 
             \State \textcolor{darkgreen}{\emph{\# of amortized model $\mathcal{G}$ for Eq.~\eqref{eq:g}} (see  \S\ref{sec:preemptive} and \S\ref{sec:meta}).}
            \State Initialize $\mathcal{G}$ from $\mathcal{F}$ and add early prediction heads.
            \State \textcolor{darkgreen}{\emph{\# (All of $\mathcal{F}$'s base parameters in $\mathcal{G}$ are frozen.)}}
            \State Train prediction heads $\mathcal{F}_{{1}\dots{l-1}}$ on $\mathcal{D}_{\mathrm{tune}}$. 
            \State Add meta early exit classifiers $\mathcal{M}_{{1}\dots{l-1}}$ to $\mathcal{G}$.
            \State \textcolor{darkgreen}{\emph{\# (All of $\mathcal{G}$'s other parameters are frozen.)}}
            \State Train meta early exit classifiers $\mathcal{M}_{{1}\dots{l-1}}$ on $\mathcal{D}_{\mathrm{meta}}$.
            \State Optionally apply temperature scaling using $\mathcal{D}_{\mathrm{scale}}$. %\textcolor{darkgreen}{\emph{\# Optional.}}
            \State \Return{$\mathcal{G}$}
        \EndFunction
        \vspace{4pt}
        \Function{Calibrate }{$\mathcal{G}$, $\mathcal{D}_{\mathrm{cal}}$, $\epsilon$}
        
            \State\textcolor{darkgreen}{\emph{\# Sets thresholds $\bm{\tau}$ of amortized model $\mathcal{G}$ for Eq.~\eqref{eq:g}}}
            \State\textcolor{darkgreen}{\emph{\# using \underline{shared calibration}~(see \S\ref{sec:calibration_shared}).}}
            \Let{$M$}{$\{\infty\}$}
            \For{$x \in \mathcal{D}_{\mathrm{cal}}$}
                \Let{$S$}{$\{\}$}
                \State\textcolor{darkgreen}{\emph{\# Record all inconsistent layers for input $x$}.}
                \State\textcolor{darkgreen}{\emph{\# Keep the highest (false) confidence score.}}
                \For{$k \in [1, l-1]$}
                    \If{$\mathcal{F}_k(x) \ne \mathcal{F}(x)$}
                        \Let{$S$}{$S \cup \mathcal{M}_k(x)$}
                    \EndIf
                \EndFor
                % \If{$|S|$ > 0}
                \Let{$M$}{$M \cup \max{(S)}$}
                % \EndIf
            \EndFor
            \State\textcolor{darkgreen}{\emph{\# Share one threshold across layers.}}
            \Let {$\tau^{\mathrm{share}}$}{$\mathrm{Quantile}\big(1 - \epsilon, M\big)$}
             \State \Return{$[\tau^{\mathrm{share}}] \times (l-1)$}
        \EndFunction
        \vspace{4pt}
        
        \Function{Predict }{$\mathcal{G}, \bm{\tau}$, $x$}
        \State\textcolor{darkgreen}{\emph{\# Implements Eq.~\eqref{eq:g}} to exit early with confidence.}
                \For{$k \in [1,l-1]$}
                      \State Compute the $k$-th prediction head of $\mathcal{G}$, $\mathcal{F}_k(x)$.
                     \If{$\mathcal{M}_k(x) > \tau_k$}%{ \Return{$\mathcal{F}_k(x)$}}  
                        \State \Return{$\mathcal{F}_k(x)$}
                     \EndIf
                \EndFor
                \State\textcolor{darkgreen}{\emph{\# Fallback to  prediction using full computation.}}
                \State \Return{$\mathcal{F}_l(x)$}
        \EndFunction
    \end{algorithmic}
    \end{algorithm}
    \end{minipage}
    \vspace{-12pt}
\end{figure}

\subsubsection{Shared calibration}\label{sec:calibration_shared} $\cset^{\mathrm{ind}}$ has the advantage of calibrating each layer independently. As $l$ grows, however, $\alpha_k$ will tend to $0$ in order to retain validity (as specified by Theorem~\ref{thm:ind}). As a result, $\cset^{\mathrm{ind}}$ will lose statistical efficiency. Following a similar approach to \citet{cauchois2020knowing} and \citet{fisch2021admission}, we compute a new test statistic, $\mathcal{M}_{\mathrm{max}}$, as
\begin{equation}
\resizebox{.89\hsize}{!}{$\displaystyle
    \hspace{-7pt}\mathcal{M}_{\mathrm{max}}(x) = \max_{k \in [l-1]}\hspace{-2pt} \left\{\mathcal{M}_k(x): \mathcal{F}_k(x) \ne \mathcal{F}(x) \right\}\hspace{-1pt}.\hspace{-2pt}
$}
\end{equation}
We discard ill-defined values when $\mathcal{M}_{\mathrm{max}}(x) = \max \varnothing$.
$\mathcal{M}_{\mathrm{max}}(x)$ reflects the \emph{worst-case} confidence across inconsistent layers for input $x$ (i.e., where $\mathcal{M}_k(x)$ predicts a high consistency likelihood for layer $k$ when layer $k$ is, in fact, \emph{inconsistent}).
This worst-case statistic allows us to keep a constant significance level $\epsilon$, even as $l$ grows. Let $m^{(1:n, \infty)}$ denote the inflated empirical distribution, \[
\resizebox{1.00\hsize}{!}{$\displaystyle
\hspace{-1pt}\{\mathcal{M}_{\mathrm{max}}(x_i) \colon x_i \in \mathcal{D}_\mathrm{cal}, \exists k ~\mathcal{F}_k(x_i) \neq \mathcal{F}(x_i)\} \cup \{\infty\}.$}
\]
We then define a single threshold shared across layers, $\tau^{\mathrm{share}} = \mathrm{Quantile}\big(1 - \epsilon, m^{(1:n, \infty)}\big)$, and predict the inconsistent index set  at $x \in \mathcal{X}$ as
\begin{equation}
\mathcal{C}_\epsilon^{\mathrm{share}}(x) = \left\{k \colon
    \mathcal{M}_k(x) \leq \tau^{\mathrm{share}}\right\}
\end{equation}
\begin{theorem}
\label{thm:share}
For any number of layers $l \in \mathbb{N}^{+}$, $\cset^{\mathrm{share}}(X_{n+1})$ is a valid set that satisfies Eq.~\eqref{eq:subset}.
\end{theorem}

%\paragraph{Conditional calibration.}
% \subsection{Conditional conformal calibration}\label{sec:cond_cal}
% Up until this point, we have been concerned with maintaining a \emph{marginal} guarantee on $\mathbb{P}(\mathcal{G}(X_{n+1}) = \mathcal{F}(X_{n+1}))$, where the randomness is over calibration points $X_{1:n}$ and test point $X_{n+1}$. The advantage of this formulation is that it only requires \emph{unlabeled} calibration data. If labeled data is available, we can calibrate our model with greater precision. In Appendix~\ref{app:cond_cal_res} we explore an extension to Eq.~\eqref{eq:marginalcoverage} that bounds the relative \emph{accuracy},
% \begin{equation}
%   \frac{\mathbb{P}(\mathcal{G}(X_{n+1}) = Y_{n+1})}{\mathbb{P}(\mathcal{F}(X_{n+1}) = Y_{n+1})} \geq 1 - \epsilon, 
% \end{equation}rather than just the prediction \emph{consistency}. For now, however, we focus on the {unlabeled} setting.

\section{Experimental Setup}\label{sec:exp}
For our main results, we use an Albert-xlarge model~\citep{Lan2020ALBERT} with 24 Transformer layers. Results using an Albert-base model and a RoBERTa-large model~\citep{liu2019roberta} are in Appendix~\ref{app:ablation}. 
% As discussed in \S\ref{sec:early_trans}, our methods can be  applied to any multilayered model such as BERT~\cite{devlin-etal-2019-bert}, GPT~\citep{brown2020language}, ResNet~\citep{he2015deep}, and others.
See Appendix~\ref{app:implement} for implementation details.
We did not search across different values for the hyper-parameters of $\mathcal{F}$ or $\mathcal{G}$ as our approach is general and guarantees consistency for any $\mathcal{F}$ with any nonconformity measure (See Appendix~\ref{app:nocomf_measures}). Tuning the hyper-parameters could further improve the efficiency of $\mathcal{G}$ while preserving consistency.

\subsection{Tasks}
We evaluate our methods on three classification tasks with varying label space size $|\mathcal{Y}|$  and difficulty: \textbf{IMDB}~\citep{maas-etal-2011-learning} sentiment analysis on movie reviews, \textbf{VitaminC}~\citep{Schuster2021} fact verification with Wikipedia articles, and \textbf{AG}~\citep{ag_news,zhang2015character} news topic classification. We also evaluate on the \textbf{STS-B}~\citep{cer-etal-2017-semeval} semantic textual similarity regression task where $\mathcal{Y}\in[0,5]\subset \mathbb{R}$. Dataset statistics, along with the test set performance of our original $\mathcal{F}$ model (Albert-xlarge), are contained in Table~\ref{tab:datasets}. 
%Recall that all of our calibration methods with the exception of \S\ref{sec:cond_cal} only require \emph{unlabeled data}.
%Dataset statistics, along with the test performance of our Albert-xlarge $\mathcal{F}$ model, are summarized in Table~\ref{tab:datasets}.

\begin{table}[!t]
    \centering
    \small
    \begin{tabular}{l|ccccc}
        \toprule
        Dataset & $|\mathcal{Y}|$ & Train & Dev. & Test & $\mathcal{F}$ test perf.\\
        \midrule
        IMDB & 2 & 20K & 5K & 25K & 94.0 \\
        VitaminC & 3 & 370K & 10K$^{*}$ & 55K  & 90.6\\
        AG News & 4 & 115K & 5K & 7.6K & 94.4 \\
        STS-B & $\infty$ & 5.7K & 1.5K & 1.4K & 89.8 \\
        \bottomrule
    \end{tabular}
    \vspace{-5pt}
    \caption{Task dataset and label space sizes. The rightmost column reports either test accuracy (classification) or Pearson-correlation (regression). $^{*}$We downsample the 63K public development set to expedite validation.}
    \label{tab:datasets}
    \vspace*{-1.5\baselineskip}
\end{table}

\subsection{Baselines} \label{sec:baselines}
In addition to our main methods discussed in \S\ref{sec:calibration}, we compare to several non-CP baselines. Note that the following methods are not  guaranteed to give well-calibrated performance (as our CP ones are).
\paragraph{Static.} We use the \emph{same} number of layers for all inputs. We choose the exit layer as the first one that obtains the desired consistency on average on $\mathcal{D}_\text{cal}$.
% We note that in finite samples ($|\mathcal{D}_{\mathrm{cal}}| < \infty$), this is not guaranteed to give calibrated performance.

\paragraph{Softmax threshold.} Following \citet{schwartz-etal-2020-right}, we exit on the first layer where $p_k^{\mathrm{max}} \ge 1-\epsilon$, where $p_k^{\mathrm{max}}$ denotes the maximum softmax response of our early classifier. Softmax values are calibrated using temperature scaling~\cite{guo-2017-deep} on another held-out (labeled) data split, $\mathcal{D}_{\mathrm{scale}}$.
% We denote this measure as \textbf{SM}. 

\paragraph{Meta threshold.} Even if perfectly calibrated, $p_k^{\mathrm{max}}$ from softmax thresholding is not measuring \emph{consistency} likelihood $\mathbb{P}(\mathcal{G}(X) = \mathcal{F}(X) \mid X = x)$, but rather $\mathbb{P}(\mathcal{G}(X) = Y \mid X = x)$. This is equivalent if $\mathcal{F}$ is an oracle, but breaks down when $\mathcal{F}$ is not. We also experiment with thresholding the confidence value of our {meta classifier} (\S\ref{sec:meta}) in a similar way (i.e., exiting when it exceeds $1 - \epsilon$).

\subsection{Evaluation}
% For each task, we use a proper training, validation, and test set. We use the training set to learn all nonconformity measures $\mathcal{S}$. We perform model selection specifically for CP on the validation set, and report final numbers on the test set. For all methods, we report the marginalized results over 25 random trials, where in each trial we partition the data into 80\% $\mathcal{D}_{\mathrm{cal}}$ ($x_{1:n}$) and 20\% $\mathcal{D}_{\mathrm{test}}$ ($x_{n+1}$).
% In order to compare the aggregate performance of different methods across all tolerance levels, we plot each metric as a function of $\epsilon$. In all plots, shaded regions show the $16$-$84$th percentiles across trials. We use the following metrics:
For each task, we use a proper training, validation, and test set. We use the training set to learn $\mathcal{F}$ and $\mathcal{G}$. We perform model selection on the validation set, and report final numbers on the test set.
For all methods, we report the marginalized results over 25 random trials, where in each trial we partition the data into 80\% $\mathcal{D}_{\mathrm{cal}}$ ($x_{1:n}$) and 20\% $\mathcal{D}_{\mathrm{test}}$ ($x_{n+1}$).  
In order to compare different methods across all tolerance levels, we plot each metric as a function of $\epsilon$. Shaded regions show the $16$-$84$th percentiles across trials. We report the following metrics:

\paragraph{Consistency.} We measure the percent of inputs for which the prediction of the CAT model $\mathcal{G}$ is the same as the full Transformer on our test prediction, i.e., $\mathcal{G}(X_{n+1}) = \mathcal{F}(X_{n+1})$.
For regression tasks, we count a prediction as consistent if it is within a small margin $\tau$ from the reference (we use $\tau = 0.5$).
%$\mathcal{G}$ is \textit{valid} if it is consistent at least $1 - \epsilon$ of the time (satisfies Eq.~\ref{eq:marginalcoverage}). 
%
As discussed in $\S\ref{sec:introduction}$, if $\mathcal{G}$ is $\epsilon$-consistent, we can also derive an average performance lower bound: it will be at least $(1 - \epsilon) \times \mathcal{F}$'s average performance.\footnote{In practice, the performance is likely to be \emph{higher} than this lower bound, since inconsistencies with $\mathcal{F}$ could lead to a correct prediction when $\mathcal{F}$ would have otherwise been wrong.} 

\paragraph{Layers ($\mathbf{\shortdownarrow}$).} We report the computational cost of the model as the average number of Transformer layers used. Our goal is to improve the efficiency (i.e., use \textit{fewer} layers) while preserving $\epsilon$-consistency. We choose this metric over absolute run-time to allow for implementation-invariant comparisons, but we provide a reference analysis next, to permit easy approximate conversions.

\begin{figure*}[!t]
\small
\centering
\footnotesize
\begin{subfigure}{0.31\textwidth}
\includegraphics[width=1.05\linewidth]{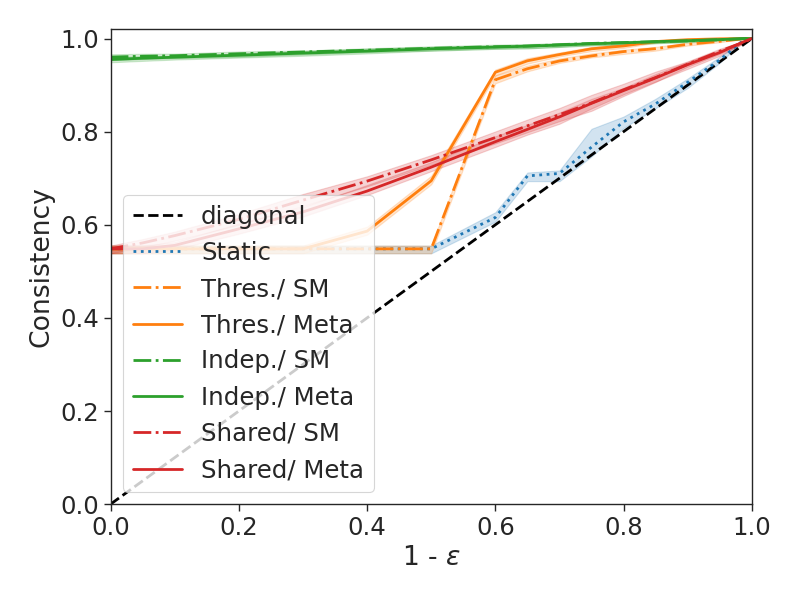} 
\end{subfigure}
%  \hfill
~
\begin{subfigure}{0.31\textwidth}
\includegraphics[width=1.05\linewidth]{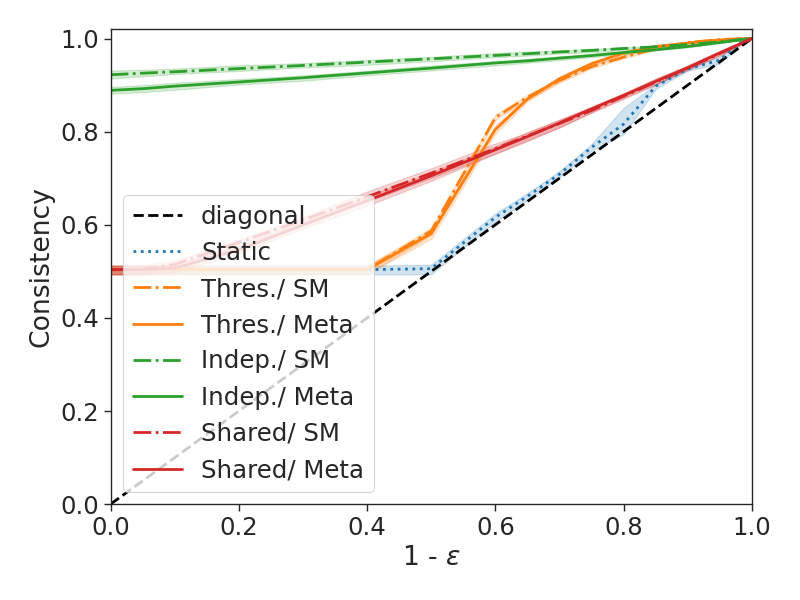}
\end{subfigure}
~
\begin{subfigure}{0.31\textwidth}
\includegraphics[width=1.05\linewidth]{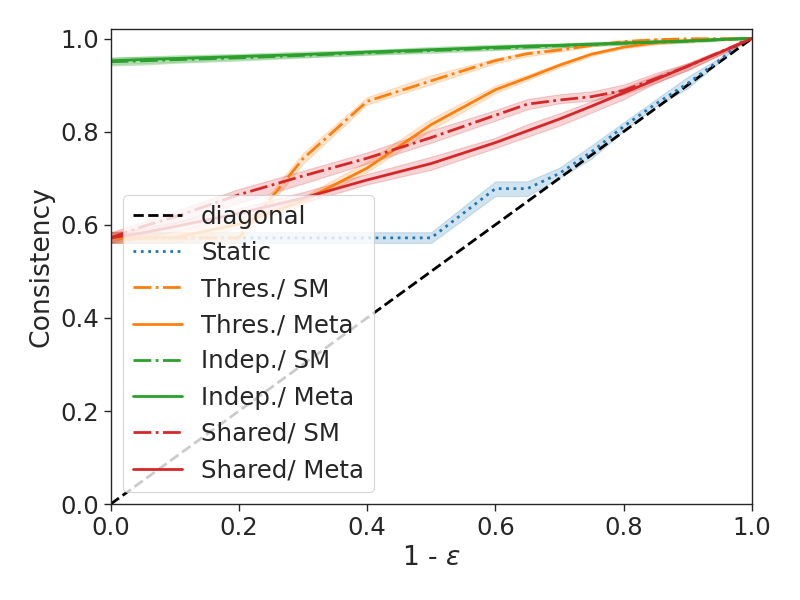} 
\end{subfigure}

\vspace*{-1.2\baselineskip}
\begin{subfigure}{0.31\textwidth}
\includegraphics[width=1.05\linewidth]{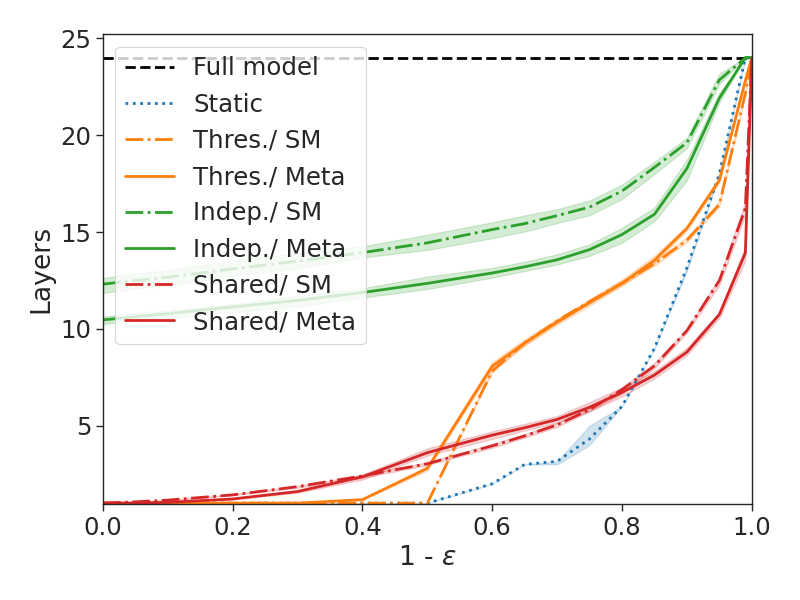} 
\vspace*{-1.8\baselineskip}
\caption{IMDB}
\end{subfigure}
~
\begin{subfigure}{0.31\textwidth}
\includegraphics[width=1.05\linewidth]{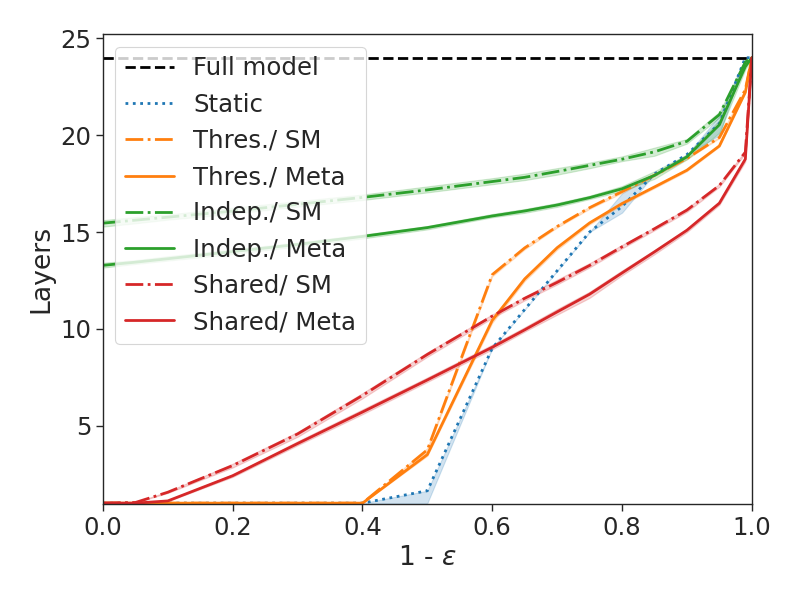}
\vspace*{-1.8\baselineskip}
\caption{VitaminC}
\end{subfigure}
~
\begin{subfigure}{0.31\textwidth}
\includegraphics[width=1.05\linewidth]{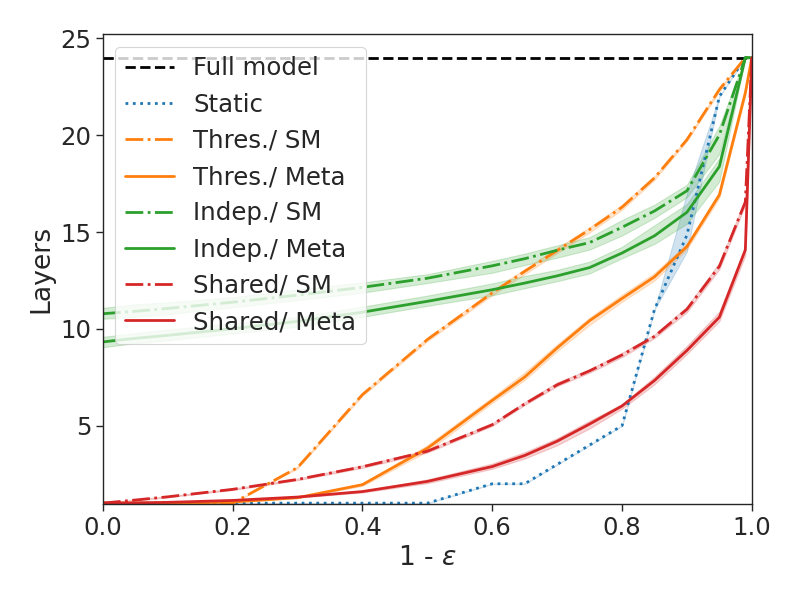} 
\vspace*{-1.8\baselineskip}
\caption{AG News}
\end{subfigure}
\vspace*{-0.5\baselineskip}

\caption{Classification results (dev). While both our CP-based methods give valid consistencies (above diagonal), \emph{shared calibration} generally results in earlier exits. This advantage is especially pronounced at smaller tolerance levels (right-hand side), where it  significantly outperforms other approaches. Our meta-learned confidence measure $\mathcal{M}_k$ improves over using the softmax response as a drop-in replacement, especially for tasks with larger $|\mathcal{Y}|$. Note that we care more about the right-hand side behavior, (i.e., larger $1-\epsilon$), as it corresponds to higher consistency.}
\vspace{-10 pt}
\label{fig:cp_res}
\end{figure*}

%\paragraph{Absolute run-time.}
\subsection{Absolute runtime analysis}
\label{sec:speed_analysis}
The exact run-time of $\mathcal{G}$ depends on the efficiency of the hardware, software, and implementation used. Ideally, the early and meta classifiers can run in parallel with the following Transformer layer (layer $k + 1$). As long as they are faster to compute concurrently than a single layer, this will avoid incurring any additional time cost. An alternative naive synchronous implementation could lead to inefficiencies when using a small tolerance $\epsilon$.

We provide a reference timing for the IMDB task implemented with the Transformers~\citep{wolf-etal-2020-transformers} library, PyTorch 1.8.1~\citep{NEURIPS2019_bdbca288}, and an A100-PCIE-40GB Nvidia GPU with CUDA 11.2. A full forward path of an Albert-xlarge takes $22.32$ms per input, $0.85$ms $\times 24$ for the transformer layers and $1.95$ms for the embedding layer and top classifier. Our early classifier takes $0.20$ms and the meta classifier takes $0.11$ms. Therefore, with a naive implementation, a CAT model $\mathcal{G}$ with an average exit layer less than $17.6$ with the meta classifier, or $19.5$ without, will realize an overall reduction in wall-clock time relative to the full $\mathcal{F}$.

We report example speedup times with the naive implementation in \S\ref{sec:res_time}, as well as an implementation invariant multiply-accumulate operation (MACs) reduction measure. The added computational effort per layer of the early predictor and meta-classifier is marginal (only $66,304$ and $1,920$ MACs, respectively). In comparison, Albert-xlarge with an input length of 256 has $\sim3 \cdot 10^{11}$ MACs.

\section{Experimental Results}
We present our main results. We experiment with both our meta classifier $\mathcal{M}_k$ confidence score (\textbf{Meta}, \S\ref{sec:meta}), and, for classification tasks, the early classifier's softmax response, $p_k^{\mathrm{max}}$ (\textbf{SM}), as a drop-in replacement for $\mathcal{M}_k$ (at no additional computational cost). Appendix \ref{app:ablation} reports results with other drop-in $\mathcal{M}_k$ replacements, in addition to results using our naive development set calibration approach (\S\ref{sec:naive}). Appendix~\ref{app:example_outputs} provides qualitative  examples. %To validate the practical efficacy of our methods, we also compare with non-CP baselines (\S\ref{sec:baselines}), using our Meta score or SM for the threshold-based approach (\textbf{Thres.}).

\subsection{Classification results}

\begin{table*}[!t]
\centering
\small

\begin{tabular}{l|ddd|ddd|ddd}
\toprule
\multicolumn{1}{c|}{Method} &
  \multicolumn{3}{c}{\textbf{IMDB}} &
  \multicolumn{3}{c}{\textbf{VitaminC}} &
  \multicolumn{3}{c}{\textbf{AG News}} \\
  &
  \multicolumn{1}{c}{Consist.} &
  \multicolumn{1}{c}{Acc.} &
  \multicolumn{1}{c}{Layers} &
  \multicolumn{1}{c}{Consist.} &
  \multicolumn{1}{c}{Acc.} &
  \multicolumn{1}{c}{Layers} &
  \multicolumn{1}{c}{Consist.} &
  \multicolumn{1}{c}{Acc.} &
  \multicolumn{1}{c}{Layers} \\
\midrule
\multicolumn{1}{l}{\textit{$\underline{1-\epsilon = 0.95}$}:} & & (88.50) & & & (86.10) & & & (89.02) &  \\[1.2mm]
\rowcolor{Gray}
Static & 95.54 & 92.88 & 18.36 & 95.51 & 89.40 & 21.00 & 95.48 & 93.20 & 22.00 \\
\rowcolor{Gray}
Thres./ SM & 99.65 & 94.01 & 16.55 & 99.83 & 90.59 & 20.07 & 100.00 & 94.44 & 22.28 \\
\rowcolor{Gray}
Thres./ Meta & 99.98 & 93.96 & 17.73 & 99.73 & 90.59 & 19.67 & 99.41 & 94.00 & 16.21 \\
% Indep./ SM & 99.72 & 93.82 & 18.05 & 99.38 & 90.16 & 20.10 & 99.81 & 94.35 & 21.76 \\
Indep./ Meta & 99.66 & 93.82 & 15.69 & 99.07 & 89.97 & 19.60 & 99.81 & 94.31 & 20.58 \\
Shared/ SM & 97.17 & 93.24 & 12.65 & 96.87 & 88.99 & 17.58 & 97.15 & 93.43 & 13.24 \\
Shared/ Meta & 97.15 & 92.71 & \mathbf{10}.\mathbf{83} & 96.91 & 89.01 & \mathbf{16}.\mathbf{79} & 97.08 & 92.50 & \mathbf{10}.\mathbf{17} \\

\midrule

\multicolumn{1}{l}{\textit{$\underline{1-\epsilon = 0.90}$}:} & & (83.84) & & & (81.57) & & & (84.33) &  \\[1.2mm]
\rowcolor{Gray}
Static & 90.82 & 89.47 & 14.00 & 92.57 & 87.80 & 19.00 & 90.88 & 89.10 & 14.00 \\
\rowcolor{Gray}
Thres./ SM & 98.88 & 93.93 & 14.71 & 99.05 & 90.27 & 18.91 & 99.68 & 94.21 & 19.53 \\
\rowcolor{Gray}
Thres./ Meta & 99.75 & 93.86 & 15.30 & 99.10 & 90.31 & 18.45 & 98.90 & 93.82 & 13.50 \\
% Indep./ SM & 99.43 & 93.67 & 17.23 & 98.89 & 89.81 & 19.47 & 99.56 & 94.10 & 18.78 \\
Indep./ Meta & 99.39 & 93.67 & 14.85 & 98.29 & 89.42 & 18.50 & 99.60 & 94.18 & 17.65 \\
Shared/ SM & 94.34 & 91.77 & 10.30 & 93.73 & 87.00 & 16.40 & 94.50 & 92.01 & 10.79 \\
Shared/ Meta & 94.36 & 90.78 & \mathbf{9}.\mathbf{01} & 93.83 & 86.89 & \mathbf{15}.\mathbf{33} & 94.29 & 90.26 & \mathbf{8}.\mathbf{35} \\

\bottomrule                               
\end{tabular}
% \vspace{-5pt}
\caption{Classification results (test) for specific tolerance levels. We report the accuracy lower bound guaranteed by our CP methods in parentheses. Shared/ Meta is reliably the most  efficient method (and is $\epsilon$-consistent). Greyed rows reflect approaches without guarantees; our CAT approaches with guarantees are presented below them.}
\vspace{-10pt}
\label{tab:main_results}
\end{table*}

Figure~\ref{fig:cp_res} summarizes the average consistency and number of layers used by $\mathcal{G}$ as a function of $\epsilon$, while Table~\ref{tab:main_results} presents results for specific $\epsilon$ on task test sets.
%
%Empirically, we confirm that both conformal methods satisfy $\epsilon$-consistency for all $\epsilon$.
% Across tasks, shared calibration provides the most significant efficiency gains---while still maintaining the desired level of consistency. 
Independent calibration proves to be quite conservative due to the loss of statistical power from the loose union bound of the Bonferroni correction for large $l$ (here $l = 24$). At some levels of $\epsilon$, non-CP baselines perform competitively, however, they lack formal guarantees. Overall, for the most critical tolerance levels (small $\epsilon$, right-hand side of the plots), our shared method leads to significant efficiency gains while still maintaining the desired level of consistency (above the diagonal). 
% , as compared to all baselines.

The effectiveness of our meta predictor, $\mathcal{M}_k$, is most pronounced for tasks with $|\mathcal{Y}|>2$, where the drop-in softmax score (SM) becomes less indicative of consistency. Both SM and Meta are relatively well-calibrated for IMDB and VitaminC, which makes the threshold-based exit rule a competitive baseline. Still, our Shared/ Meta method provides both reliable and significant gains.

% test table
%Table~\ref{tab:main_results} shows the results for the test set for two reference tolerance levels ($\epsilon=0.05,0.1$). From the consistency guarantee of our CP methods, we also derive a lower bound on the classification accuracy of $\mathcal{G}$ as it should keep at least $1-\epsilon$ of the correct predictions of $\mathcal{F}$. We report this lower bound per task in parenthesis in the table. In practice, the accuracy is likely to be higher than guaranteed since inconsistencies on mistakes of $\mathcal{F}$ could lead to correct predictions.

% Our meta predictor combined with shared calibration provides a useful tool for reducing the computational cost while effectively controlling the consistency tolerance.
The computational advantage of our CAT model is dependent on the average difficulty of the task and the implementation. As Table~\ref{tab:main_results} shows, allowing up to an $\epsilon$ of 10\% inconsistency, for two of the tasks we cut down the average Transformer layer to only $9$ out of $24$ using our Shared/ Meta model. 
% Following the analysis in \S\ref{sec:speed_analysis}, 
This leads to an approximate speedup of $1.8\times$ with a synchronous implementation and of $2.7\times$ with a concurrent one, compared to running the full model.
Moreover, Figure~\ref{fig:imdb_hist} illustrates the user's control over available computational resources via modulating $\epsilon$. 
%Importantly, $\epsilon$ can be easily modified during inference, without having to retrain any model parameters.
Decreasing $\epsilon$ increases the confidence level required before committing to the early classifier's prediction (thereby increasing the average number of required layers), and vice-versa.

\begin{figure}[!t]
\small
\centering
\footnotesize
\vspace*{-0.8\baselineskip}
\begin{subfigure}{0.36\textwidth}
\includegraphics[width=1.05\linewidth]{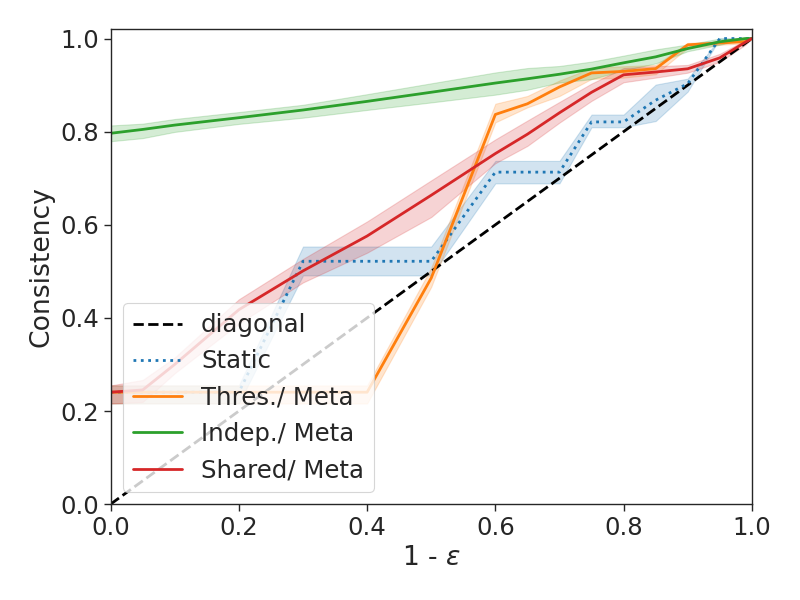} 
\end{subfigure}

\vspace*{-1.3\baselineskip}
\begin{subfigure}{0.36\textwidth}
\includegraphics[width=1.05\linewidth]{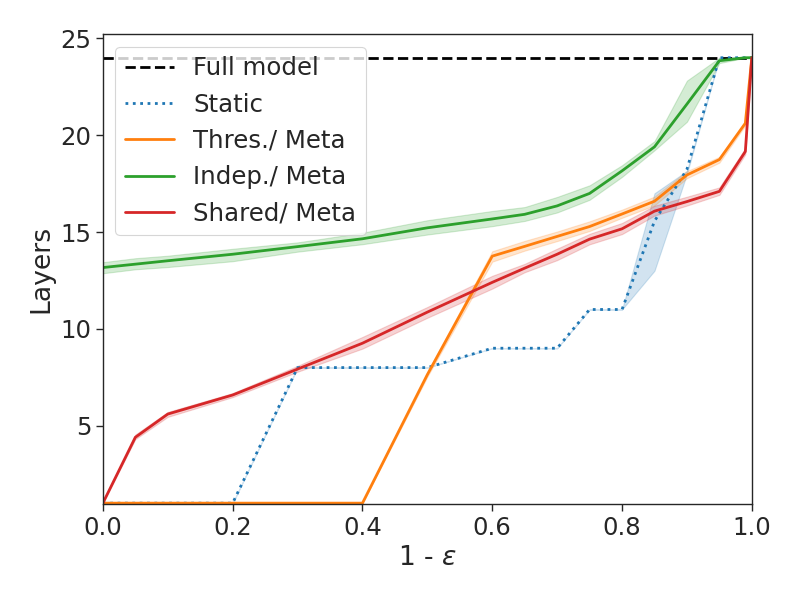} 
\vspace*{-1.8\baselineskip}
\end{subfigure}
\vspace*{-0.5\baselineskip}
\caption{Dev results for the STS-B regression task.}% As in classification (Table~\ref{fig:cp_res}), the advantage of our Shared/ Meta approach is especially pronounced at smaller tolerance levels (right-hand side, which we typically care more about), where it outperforms other approaches.}
%\vspace{-10pt}
\vspace*{-1.2\baselineskip}
\label{fig:cp_reg_res}
\end{figure}

\subsection{Regression results}
Table~\ref{tab:regression_res} and Figure~\ref{fig:cp_reg_res} present results for our regression task, where we see similar trends. Here, an attractive advantage of our meta confidence predictor is its generalizability to multiple task output types. Notice that the event space of $\mathbf{1}\{\mathcal{G}(X) = \mathcal{F}(X)\} = \{0, 1\}$ always, regardless of the original $\mathcal{Y}$.\footnote{As long as equality is  suitably defined, e.g., for STS-B we define consistent outputs as being within $\tau=0.5$ away.} This allows it to be easily adapted to tasks beyond classification, such as regression, where traditional softmax-based confidence measures (as used in, e.g., \citet{schwartz-etal-2020-right}) are absent.
%One attribute of interest of our meta confidence predictor, is its self-contained nature. This allows an easy adoption to other tasks, such as regression, where softmax-based confidence measures are absent. Figure~\ref{fig:cp_reg_res} depicts the results for the STSB regression task of predicting a similarity score for a pair of sentence on a $0-5$ scale. We consider two predictions as consistent if they are within a $\tau=0.5$ interval. 

%The static baseline becomes ineffective for small tolerance levels and reduces to using the full model. Our meta measure provides an effective confidence estimate. Thresholding directly on the meta value, however, can violate validity (consistency under the diagonal) as it is not provably calibrated. Similar to the classification results, our shared method provides both provable guarantees and improved empirical results for small tolerance values.

%Table~\ref{tab:regression_res} reports the results for STSB's test set for two values of $\epsilon$.

\begin{figure*}[!t]
    \centering
    \includegraphics[width=1\linewidth]{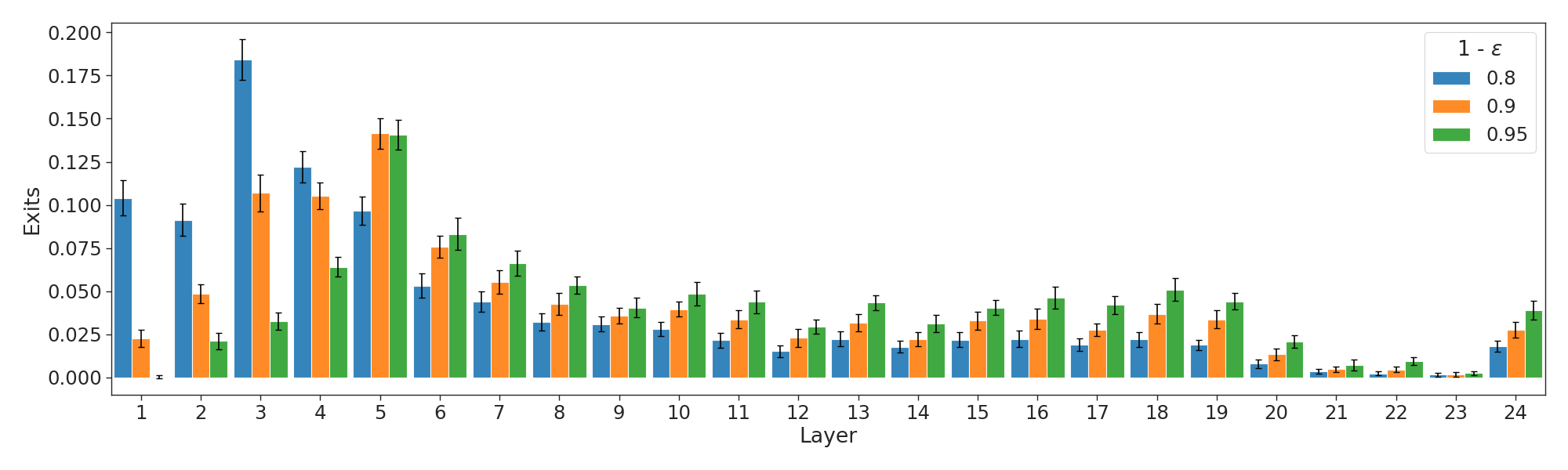} 
    \vspace{-27pt}
    \caption{Distribution of exit layers per tolerance level $\epsilon$ for the IMDB task (dev set) with Shared/ Meta. Larger $\epsilon$ allows the CAT model to shift its predictions earlier by permitting for more inconsistencies with the full model $\mathcal{F}$.}
    \label{fig:imdb_hist}
    %\vspace*{-10pt}
\end{figure*}

\begin{table}[!t]
\centering
\small
% \setlength{\tabcolsep}{10pt}
% \vspace{-7pt}
% \vspace*{-0.8\baselineskip}
\begin{tabular}{l|dd}
\toprule
\multicolumn{1}{c|}{Method} &
  \multicolumn{1}{c}{Consist.} &
  \multicolumn{1}{c}{Layers} \\
\midrule
\multicolumn{1}{l}{\textit{$\underline{1-\epsilon = 0.95}$}:}  \\
\rowcolor{Gray}
Static & 100.00 & 24.00  \\
\rowcolor{Gray}
Thres./ Meta & 99.87 & 19.19  \\
Indep./ Meta & 99.29 & 23.60  \\
Shared/ Meta & 96.42 & \mathbf{17}.\mathbf{64}  \\
\midrule

\multicolumn{1}{l}{\textit{$\underline{1-\epsilon = 0.90}$}:} \\
\rowcolor{Gray}
Static & 92.51 & 20.00  \\
\rowcolor{Gray}
Thres./ Meta & 99.19 & 18.53  \\
Indep./ Meta & 97.77 & 20.26  \\
Shared/ Meta & 92.65 & \mathbf{17}.\mathbf{29}  \\

\bottomrule                               
\end{tabular}
% \vspace{-5pt}
\caption{Test results for the STS-B regression task.}% Greyed rows reflect approaches without guarantees.}
\vspace{-10pt}
\label{tab:regression_res}
\end{table}

\begin{table*}[!t]
\centering
\small
\vspace{-4pt}

\begin{tabular}{l|dddd|dddd}
\toprule
\multicolumn{1}{c|}{Method} &
  \multicolumn{4}{c|}{\textbf{Amortized time ($100\cdot T_{\mathcal{G}} / T_{\mathcal{F}}$})} &
  \multicolumn{4}{c}{\textbf{MACs reduction} ($|\mathcal{F}| / |\mathcal{G}|$)} \\
  \multicolumn{1}{c|}{} & 
  \multicolumn{1}{c}{IMDB} &
  \multicolumn{1}{c}{VitaminC} &
  \multicolumn{1}{c}{AG News} &
  \multicolumn{1}{c|}{STS-B} &
  \multicolumn{1}{c}{IMDB} &
  \multicolumn{1}{c}{VitaminC} &
  \multicolumn{1}{c}{AG News} &
  \multicolumn{1}{c}{STS-B} \\
\midrule
% \multicolumn{1}{l}{\textit{$\underline{1-\epsilon = 0.95}$}:} &  \\[1.2mm]
% \rowcolor{Gray}
% Thres./ SM & 
% 85.56   &	102.12  &  	112.52  &	\multicolumn{1}{c|}{\text{N/A}}       &   1.45    &	1.20    &	1.08    &	\multicolumn{1}{r}{\text{N/A}}   \\
% \rowcolor{Gray}
% Thres./ Meta & 
% 99 .85   &	109.93  &  	91.95   &	107.44   &	1.35    &	1.22    &	1.48    &	1.25    \\ 
% Indep./ Meta &
% 89.25   &	109.57  &  	114.66  &	130.36  &	1.53    &	1.22    &	1.17    &	1.02    \\
% Shared/ SM &
% 67.22   &	\textbf{90}.\textbf{41}   &  	69.99   &	\multicolumn{1}{c|}{\text{N/A}}      &	1.90    &	1.37    &	1.81    &	\multicolumn{1}{r}{\text{N/A}}       \\
% Shared/ Meta &
% \textbf{63} .\textbf{99} \ \%   &	94.97 \ \%   &  	\textbf{60}.\textbf{56} \ \%   &	\textbf{99}.\textbf{38} \ \%   &	\times \textbf{2}.\textbf{22}    &	\times \textbf{1}.\textbf{43}    &	\times \textbf{2}.\textbf{36}    &	\times \textbf{1}.\textbf{36}   \\

% \midrule
% \multicolumn{1}{l}{\textit{$\underline{1-\epsilon = 0.90}$}:} &  \\[1.2mm]
\rowcolor{Gray}
Thres./ SM & 
76.91   &	96.66   &	99.58   &	\multicolumn{1}{c|}{\text{N/A}}       &   1.63    &	1.27    &	1.23    &	\multicolumn{1}{r}{\text{N/A}}   \\
\rowcolor{Gray}
Thres./ Meta & 
87.22   &	103.59  &	77.87   &	104.01  &	1.57    &	1.30    &	1.78    &	1.30    \\
Indep./ Meta &
84.88   &	103.85  &	99.44   &	113.00  &	1.62    &	1.30    &	1.36    &	1.18    \\
Shared/ SM &
56.16   &	\textbf{84}.\textbf{86}   &	58.47   &	\multicolumn{1}{c|}{\text{N/A}}       &	2.33    &	1.46    &	2.22    &	\multicolumn{1}{r}{\text{N/A}}   \\
Shared/ Meta &
\textbf{54}.\textbf{53} \ \%   &	87.38  \ \%  &	\textbf{51}.\textbf{10} \ \%   &	\textbf{97}.\textbf{56} \ \%   &	\times \textbf{2}.\textbf{66}    &	\times \textbf{1}.\textbf{57}    &	\times \textbf{2}.\textbf{87}    &	\times \textbf{1}.\textbf{39}    \\

\bottomrule                               
\end{tabular}
\caption{Reference time speedup and model complexity reduction for $1-\epsilon=0.90$ (see Table~\ref{tab:speedup_results_app} for 0.95). We compute the amortized time with the naive synchronous implementation (\S\ref{sec:speed_analysis}). A more efficient implementation can further reduce the time of $\mathcal{G}$. The MACs reduction measure is implementation agnostic and expresses the ratio of computational effort saved by $\mathcal{G}$. Our CAT models (non-greyed lines) not only guarantee $1-\epsilon$ consistency with $\mathcal{F}$, but are also significantly more efficient in practice when using Shared calibration.}

\vspace{-9pt}
\label{tab:speedup_results}
\end{table*}

\subsection{Example efficiency gains} \label{sec:res_time}
% \textcolor{red}{TODO describe table and define MACs}
Following the analysis in \S\ref{sec:speed_analysis}, we compute the amortized inference time with a naive implementation and report its percentage out of the full model. As Table~\ref{tab:speedup_results} shows, our Shared calibration is the most efficient method on all four tasks.
For tasks with many easy inputs (IMDB and AG News), our Shared/ Meta method can save $45\%$ - $49\%$ of the inference time when $1-\epsilon=0.90$. 
Unsurprisingly, the absolute speedup is less significant for harder tasks, but increases with higher tolerance levels.
% In the harder VitaminC and STS-B tasks, the speedup is smaller since the sufficient confidence arrives on average at later layers. 

On VitaminC, even though the Meta measure allows exiting on earlier layers, its additional meta classifiers result in slightly slower inference on average at this tolerance level, compared to our Shared/ SM. With a more efficient concurrent implementation, the Meta measure will be favorable. 

We also compute the MACs reduction metric which is independent of the specific implementation or hardware and shows the number of multiply-accumulate operations of the full model compared to our CAT model. As demonstrated in Table~\ref{tab:speedup_results}, our Shared/ Meta method is most effective in reducing the computational effort across all tasks for the two examined tolerance levels.

\section{Conclusion}\label{sec:conclusion}
The ability to make predictions quickly without excessively degrading performance is critical to production-level machine learning systems. 
In fact, being capable of quantifying the uncertainty in a prediction and deciding when additional computation is needed (or not) is a key challenge for any intelligent system (e.g., see the System 1 vs. System 2 dichotomy explored in \citet{kahneman2011thinking}). %of .. This tension is described in the system 1 vs.\ system 2 dichotomy of \citet{kahneman2011thinking} for human's behaviors. System 1 is fast and relies on intuitions while system 2 is slow and effortful but more accurate. A key dilemma of system 2 is in deciding when to override the intuitive prediction of system 1 and to allocate additional resources.
%\emph{when} to be certain enough in a prediction and when an additional effort is needed, is a crucial challenge for any intelligent system. This tension is described in the system 1 vs.\ system 2 dichotomy of \citet{kahneman2011thinking} for human's behaviors. System 1 is fast and relies on intuitions while system 2 is slow and effortful but more accurate. A key dilemma of system 2 is in deciding when to override the intuitive prediction of system 1 and to allocate additional resources.

In this work, we addressed the crucial challenge of deciding \emph{when} to sufficiently trust an early prediction of Transformer-based  models by learning from their past predictions. Our Confident Adaptive Transformers (CATs) framework leverages \emph{meta predictors} to accurately assess whether or not the prediction of a simple, early classifier trained on an intermediate Transformer representation is likely to already be consistent with that of the \emph{full} model $\mathcal{F}(X)$ (i.e., after all $l$ layers of $\mathcal{F}$ are computed). Importantly, we develop a new conformal prediction approach for calibrating the confidence of the meta classifier that is (1) simple to implement, (2) fast to compute alongside the Transformer, (3) requires only \emph{unlabeled} data, and (4) provides statistically efficient marginal guarantees on the event that the prediction of the faster, amortized CAT model is \emph{consistent} with that of the full $\mathcal{F}$. %the \emph{consistency} of the amortized, faster CAT model with the full one.
Our results on multiple tasks demonstrate the generality of our approach, and its effectiveness in consistently improving computational efficiency---all while maintaining a reliable margin of error. 
%\nopagebreak
%performance metrics. % that in reducing the average computational effort by dynamically allocating the model's complexity and provably maintaining consistency.

%the likelihood that an early prediction is yet to be consistent with the full model.include early classifiers and meta predictors to accurately assess the likelihood that an early prediction is yet to be consistent with the full model. Importantly, we develop a new conformal prediction approach for calibrating the confidence measure to provide marginal guarantees on the consistency of the amortized, faster model with the full one. Our results demonstrate the effectiveness of our method in reducing the average computational effort by dynamically allocating the model's complexity and provably maintaining consistency.

% TODO: uncomment
\section*{Acknowledgements}
We thank the MIT NLP group for helpful discussions.
TS is supported in part by DSO grant DSOCL18002. AF is supported in part by an NSF GRFP. This work is also supported in part by DSTA award DST00OECI20300823.

\bibliography{anthology,emnlp2020}
\bibliographystyle{acl_natbib}
% \clearpage
\appendix
\counterwithin{figure}{section}
\counterwithin{table}{section}
\section{Proofs}
\label{app:proofs}

We first state the following useful lemma on inflated sample quantiles.

\begin{lemma}
\label{lemma:quantile}
Let $\mathrm{Quantile}(\alpha; F)$ denote the $\alpha$ quantile of distribution $F$. Let $V_{1:n}$ denote the empirical distribution over random variables $\{V_1, \ldots, V_n\}$. Furthermore, assume that $V_i$, $i=1,$ $\ldots, n+1$ are exchangeable.
Then for any $\alpha \in (0, 1)$,  we have $\mathbb{P}\left(V_{n+1} \leq \mathrm{Quantile}(\alpha, V_{1:n} \cup \{\infty\})\right) \geq \alpha.$
\end{lemma}

\begin{proof}
This is a well-known result. Given support points $v_1, \ldots, v_n \in \mathbb{R}$ for a discrete distribution $F$, let $q = \mathrm{Quantile}(\alpha; F)$. Any points $v_i > q$ do not affect this quantile, i.e., if we consider a new distribution $\tilde{F}$ where all points $v_i > q$ are mapped to arbitrary values also larger than $q$ then $\mathrm{Quantile}(\alpha; F) = \mathrm{Quantile}(\alpha; \tilde{F})$. Accordingly, for the exchangeable $V_i$, we have 
\begin{align*}
    V_{n+1} &> \mathrm{Quantile}(\alpha; V_{1:n} \cup \{\infty\}) \Longleftrightarrow \\
    &V_{n+1} > \mathrm{Quantile}(\alpha; V_{1:(n+1)}).
\end{align*}
Equivalently, we also have that
\begin{align*}
    V_{n+1} &\leq \mathrm{Quantile}(\alpha; V_{1:n} \cup \{\infty\}) \Longleftrightarrow \\
    &V_{n+1} \leq \mathrm{Quantile}(\alpha; V_{1:(n+1)}).
\end{align*}
Given the discrete distribution over the $n+1$ variables $V_i$, $V_{n+1} \leq \mathrm{Quantile}(\alpha; V_{1:(n+1)})$ implies that $V_{n+1}$ is among the $\lceil\alpha(n+1)\rceil$ smallest of $V_{1:(n+1)}$.
By exchangeability, this event occurs with probability at least $\frac{\lceil\alpha(n+1)\rceil}{n+1} \geq \alpha$. \end{proof}

\subsection{Proof of Proposition~\ref{prop:naive}}
\begin{proof}
 This result is based on Clopper-Pearson confidence interval for Binomial random variables~\cite{clopperpearson}. As the binary events $\mathbf{1}\{\mathcal{G}(X_i; \bm{\tau}) = \mathcal{F}(X_i)\}$ are i.i.d., the sum $s$ is Binomial. Directly applying a one-sided Clopper-Pearson lower bound on the true success rate, $\mathbb{P}(\mathcal{G}(X_i; \bm{\tau}) = \mathcal{F}(X_i))$, gives the result.
\end{proof}

\subsection{Proof of Proposition~\ref{prop:early_exit}}
\begin{proof}
We prove by simple calculation using the property assumed in Eq.~\eqref{eq:subset}.
\begin{align*}
    \mathbb{P}(\mathcal{F}_K(&X_{n+1}) = \mathcal{F}(X_{n+1})) \\
    &= \mathbb{P}(\min \cset^c(X_{n+1}) \in \mathcal{I}^c(X_{n+1})) \\
    &\geq \mathbb{P}(\cset^c(X_{n+1}) \subseteq \mathcal{I}^c(X_{n+1})) \\
    &= \mathbb{P}(\mathcal{I}(X_{n+1}) \subseteq \cset(X_{n+1})) \\
    &\geq 1 - \epsilon.
\end{align*}
\end{proof}

\subsection{Proof of Theorem~\ref{thm:ind}}
\begin{proof}
 For a given $k$, let $V_k^{(i)} := \mathcal{M}_k(X_i)$ denote the random meta confidence values used for calibration, and $V_k^{(n+1)} := \mathcal{M}_k(X_{n+1})$ the random test point. For all $k$, $\mathcal{M}_k$ is trained and evaluated on separate data ($\mathcal{D}_{\mathrm{meta}}$ vs $\mathcal{D}_{\mathrm{cal}} \cup \mathcal{D}_{\mathrm{test}}$), preserving exchangeability. Therefore, as $X_{1:n+1}$ are exchangeable, then $V_{k}^{(1:n+1)}$ are also exchangeable.
 
 Layer $k$ is included in $\cset^{\mathrm{ind}}$ iff $V_k^{(n+1)} \leq \mathrm{Quantile}(1 - \alpha_k, V_k^{(1:n)} \cup \{\infty\})$. For a given $k$, this happens with probability at least  $1 - \alpha_k$ by Lemma~\ref{lemma:quantile}. Taken over all $k \in \mathcal{I}(X_{n+1})$ where $|\mathcal{I}(X_{n+1})|$ is at most $l - 1$ (i.e., \emph{all} early layers are inconsistent), we have
 \begin{align*}
     \mathbb{P}(\mathcal{I}(&X_{n+1}) \subseteq \cset^{\mathrm{ind}}(X_{n+1})) \\
     &= 1 - \mathbb{P}\Big(\bigcup_{k \in \mathcal{I}} \{k \not\in \cset^{\mathrm{ind}}(X_{n+1})\}\Big) \\
     &\geq 1 - \sum_{k \in \mathcal{I}} \mathbb{P}(k \not\in \cset^{\mathrm{ind}}(X_{n+1}) \\
     &= 1 - \sum_{k \in \mathcal{I}} \alpha_k \\
     &\geq 1 - \epsilon.
 \end{align*}
The last inequality is given by the Bonferroni constraint, i.e., $\alpha_k = \omega_k \cdot \epsilon$, where $\sum_{i=1}^{l -1}\omega_i = 1$
\end{proof}

\subsection{Proof of Theorem~\ref{thm:share}}
\begin{proof}
 By the same argument as Theorem~\ref{thm:ind}, the meta scores $\mathcal{M}_k(X_i)$ are exchangeable. Since $\mathcal{M}_{\mathrm{max}}$ operates symmetrically across all $X_i$, $M^{(i)} = \mathcal{M}_{\mathrm{max}}(X_i)$ are also exchangeable.
 
 Let $M^{(n+1)}$ denote the maximum meta score across inconsistent layers for the new test point. By Lemma~\ref{lemma:quantile}, this falls below $\mathrm{Quantile}(1 - \epsilon,M^{(1:n)} \cup \{\infty\})$ with probability at least $1 - \epsilon$. Since $M^{(n+1)}$ reflects the maximum meta score, this entails that the meta scores of all other inconsistent layers $k \in \mathcal{I}(X_{n+1})$ for $X_{n+1}$ will be below $\mathrm{Quantile}(1 - \epsilon,M^{(1:n)} \cup \{\infty\})$ if $M^{(n+1)}$ is, and thereby be included in $\cset^\mathrm{share}(X_{n+1})$. This gives the bound in Eq.~\eqref{eq:subset}.
 \end{proof}

% \subsection{Proof of Corollary~\ref{cor:cond}}
% \begin{proof}
%  It suffices to show that the filter $\mathcal{A} := \{X_i \colon \mathcal{F}(X_i) = Y_{i}\}$ produces exchangeable samples $X_j \in \mathcal{A}$. The condition $\mathcal{F}(X_i) = Y_i$ is symmetric across all $X_i$, where $\mathcal{F}$ is fixed. Therefore the subset $\mathcal{A}$ is also exchangeable conditioned on $\mathcal{F}(X_i) = Y_{i}$. The setting then reduces to a straightforward application of Proposition~\ref{prop:early_exit}.
% \end{proof}
\begin{table*}[!t]
\renewcommand\thetable{C.1}
\centering
\small
\resizebox{2\columnwidth}{!}{%

\begin{tabular}{c|ddd|ddd|ddd}
\toprule
\multicolumn{1}{c}{Nonconformity} &
  \multicolumn{3}{c}{\textbf{IMDB}} &
  \multicolumn{3}{c}{\textbf{VitaminC}} &
  \multicolumn{3}{c}{\textbf{AG News}} \\
  \multicolumn{1}{c}{measure}&
  \multicolumn{1}{c}{Consist.} &
  \multicolumn{1}{c}{Bound} &
%   \multicolumn{1}{c}{Acc.} &
  \multicolumn{1}{c}{Layers} &
  \multicolumn{1}{c}{Consist.} &
  \multicolumn{1}{c}{Bound} &
%   \multicolumn{1}{c}{Acc.} &
  \multicolumn{1}{c}{Layers} &
  \multicolumn{1}{c}{Consist.} &
  \multicolumn{1}{c}{Bound} &
%   \multicolumn{1}{c}{Acc.} &
  \multicolumn{1}{c}{Layers}
  \\
\midrule
% \multicolumn{2}{l}{\textit{$\underline{1-\epsilon = 0.95}$}:} & & (88.50) & & & & (85.17) & & & & (89.02) & \\[1.2mm]
\multicolumn{2}{l}{\textit{$\underline{1-\epsilon = 0.95}$}:} &  & & &  & & &  & \\[1.2mm]
SM & 95.16 & 93.74 & 10.39 & 94.84 & 94.04 & 16.60 & 95.02 & 93.75 & 11.63 \\
Meta & 94.96 & 93.72 & 9.13 & 94.93 & 94.12 & 15.60 & 94.86 & 93.58 & 9.37 \\
\midrule
\multicolumn{2}{l}{\textit{$\underline{1-\epsilon = 0.9}$}:} & &  & & &  & & & \\[1.2mm]
SM &  90.22 & 88.30 & 7.35 & 89.85 & 88.59 & 14.93 & 89.72 & 88.01 & 8.98 \\
Meta &  90.19 & 88.36 & 7.13 & 90.00 & 88.70 & 13.67 & 90.14 & 88.48 & 6.85 \\

\bottomrule                               
\end{tabular}
}%

\caption{Results (dev) using the naive development set calibration method (see \S\ref{sec:naive}). This method tunes the early exit thresholds to get efficient $\epsilon$-consistent predictions on a development set, but does not guarantee that prediction will be $\epsilon$-consistent on new data. ``Consist.'' measures the empirical consistency on a test set, from which we compute a guaranteed lower bound (``Bound'') to 99\% confidence.  The bound is significantly lower than our target $1-\epsilon$, and the measured consistency in our experiments also falls slightly bellow $1-\epsilon$ in some cases.}
\label{tab:test_bound}
\end{table*}

\section{Implementation Details}\label{app:implement}

We implement our early exit Transformers (\S\ref{sec:early_trans}) on top of the Transformers library~\citep{wolf-etal-2020-transformers}.\footnote{As discussed in \S\ref{sec:early_trans}, our methods can also be  applied to any multilayered model such as BERT~\cite{devlin-etal-2019-bert}, GPT~\citep{brown2020language}, ResNet~\citep{he2015deep}, and others.} We set $d_e$ to 32 in our experiments. For each task we fix a pre-trained $\mathcal{F}$ and train the early and meta classifiers. 
We reuse the same training data that was used for $\mathcal{F}$ and divide it to 70/10/20\% portions for $\mathcal{D}_{\text{tune}},\mathcal{D}_{\text{scale}}$ and $\mathcal{D}_{\text{meta}}$, respectively.
For classification tasks, we add the temperature scaling step~\citep{pmlr-v70-guo17a} after the early training to improve the calibration of the softmax. We run the scaling for 100 steps on $\mathcal{D}_{\text{scale}}$ using an Adam optimizer~\citep{kingma2017adam} with a learning rate of $10^{-3}$. For the early and meta training we use the same optimizer as for $\mathcal{F}$. 

We fix $\mathcal{F}$ rather than train it jointly with the new components of $\mathcal{G}$ to avoid 
 any reduction in $\mathcal{F}$'s performance~\citep{xin-etal-2020-deebert}. This also makes our method simple to train over any existing Transformer without having to retrain the whole model which could be very costly. Training all parameters of $\mathcal{G}$ jointly can lead to more efficient inference as the early representations will be better suited for classification~\citep{schwartz-etal-2020-right, geng2021romebert}, but potentially with the cost of reducing the accuracy of $\mathcal{F}_l$. In the case of joint training, our CATs will provide consistency guarantees with respect to the jointly-trained $\mathcal{F}_l$.
 
 We implement the conformal calibration process in Python and perform retrospective analysis with different random splits of  $\mathcal{D}_{\text{cal}}$ and  $\mathcal{D}_{\text{test}}$. For Theorem~\ref{thm:ind}, we simply use the uniform 
 Bonferroni correction, setting $w_k = \frac{1}{l-1} \quad \forall k$. 
 For the naive development set calibration, we use a shared threshold across all layers in order to reduce the examined solution space in Equation~\ref{eq:dev_set_thres}.

\section{Additional Results}\label{app:ablation}

In this section, we provide complementary results for the experiments in the main paper. All results, except for sections \ref{app:base_res} and \ref{app:roberta_res}, are with an Albert-xlarge model as $\mathcal{F}$, similar to the main paper. However, we note that the results in these tables are based on the development sets, while the tables in the main paper report the test set results.

% \subsection{Regression development results}\label{app:reg_dev_res}
% Figure~\ref{fig:cp_reg_res} shows the development results on the regression STSB task. 

% \begin{figure}[!t]
% \small
% \centering
% \footnotesize
% \begin{subfigure}{0.4\textwidth}
% \includegraphics[width=1.05\linewidth]{figures/stsb/const.png} 
% \end{subfigure}

% \vspace*{-1.2\baselineskip}
% \begin{subfigure}{0.4\textwidth}
% \includegraphics[width=1.05\linewidth]{figures/stsb/size.png} 
% % \vspace*{-1.8\baselineskip}
% \end{subfigure}
% \vspace*{-0.5\baselineskip}
% \caption{Regression results for STSB (dev).}
% % \vspace{-10pt}
% \vspace*{-1\baselineskip}
% \label{fig:cp_reg_res}
% \end{figure}

\begin{table*}[!t]
\renewcommand\thetable{C.2}
\centering
\small

\begin{tabular}{l|ddd|ddd|ddd}
\toprule
\multicolumn{1}{c}{Nonconformity} &
  \multicolumn{3}{c}{\textbf{IMDB}} &
  \multicolumn{3}{c}{\textbf{VitaminC}} &
  \multicolumn{3}{c}{\textbf{AG News}} \\
 \multicolumn{1}{c}{measure} &
  \multicolumn{1}{c}{Consist.} &
  \multicolumn{1}{c}{Acc.} &
  \multicolumn{1}{c}{Layers} &
  \multicolumn{1}{c}{Consist.} &
  \multicolumn{1}{c}{Acc.} &
  \multicolumn{1}{c}{Layers} &
  \multicolumn{1}{c}{Consist.} &
  \multicolumn{1}{c}{Acc.} &
  \multicolumn{1}{c}{Layers} \\
\midrule
\multicolumn{1}{l}{\textit{$\underline{1-\epsilon = 0.95}$}:} & & (88.50) & & & (85.17) & & & (89.02) & \\[1.2mm]
Random & 97.23 & 91.56 & 21.57 & 96.91 & 87.42 & 22.71 & 97.11 & 91.58 & 21.60 \\
$D_{\mathrm{KL}}(p_{k-1}||p_k)$ & 97.36 & 92.49 & 19.33 & 96.84 & 88.85 & 22.28 & 97.08 & 92.46 & 20.18 \\
$\mathcal{H}(p_k)$  & 97.28 & 92.84 & 12.49 & 96.79 & 88.28 & 17.44 & 97.15 & 92.79 & 14.55 \\
$p_k^{\mathrm{diff}}$  & 97.28 & 92.84 & 12.49 & 96.83 & 88.38 & 17.42 & 96.96 & 92.80 & 12.89 \\
$p_k^{\mathrm{max}}$ (SM)    & 97.28 & 92.84 & 12.49 & 96.79 & 88.31 & 17.40 & 97.08 & 92.81 & 13.23 \\
Meta & 96.99 & 92.24 & 10.75 & 96.91 & 88.29 & 16.49 & 96.98 & 91.98 & 10.60 \\

\midrule

\multicolumn{1}{l}{\textit{$\underline{1-\epsilon = 0.90}$}:} & & (83.84) & & & (80.69) & & & (84.33) &  \\[1.2mm]
Random & 94.52 & 89.68 & 19.21 & 93.94 & 85.44 & 21.47 & 94.27 & 89.28 & 19.01 \\
$D_{\mathrm{KL}}(p_{k-1}||p_k)$ & 94.48 & 91.36 & 12.13 & 93.76 & 86.81 & 20.49 & 93.88 & 89.98 & 14.59 \\
$\mathcal{H}(p_k)$ & 94.49 & 91.31 & 9.91 & 93.67 & 86.41 & 16.29 & 94.54 & 90.80 & 13.08 \\
$p_k^{\mathrm{diff}}$ & 94.49 & 91.31 & 9.91 & 93.67 & 86.53 & 16.11 & 94.02 & 90.56 & 10.69 \\
$p_k^{\mathrm{max}}$ (SM) & 94.49 & 91.31 & 9.91 & 93.68 & 86.44 & 16.13 & 94.05 & 90.76 & 11.01 \\
Meta & 94.40 & 90.45 & 8.80 & 93.74 & 86.17 & 15.09 & 94.08 & 89.72 & 8.88 \\

\bottomrule                               
\end{tabular}

\caption{Results (dev) of our Shared model on the classification tasks using different nonconformity measures. $p_k^{\mathrm{diff}}$ and $p_k^{\mathrm{max}}$ are defined in Table~\ref{tab:features}, $D_{\mathrm{KL}}(p_{k-1}||p_k)$ is the Kullback-Leibler Divergence between the previous layer's softmax outputs and the current layer, and $\mathcal{H}(p_k)$ is the entropy of the softmax outputs. Our CP-based Shared method provides the guaranteed consistency with any measure, even random. The benefit, however, of using a better measure is in confidently exiting earlier. Our Meta measure allows the use of least Transformer layers meeting the consistency requirement with enough confidence.}
\label{tab:ablation_noncomf}
\end{table*}

\subsection{Naive development set calibration}

For completeness, we evaluate the simple, but naive, calibration method described in \S\ref{sec:naive}. Recall that in this approach we first tune $\bm{\tau}$ on a development set, and then bound the resulting $\mathcal{G}$'s accuracy using another heldout calibration split. The bound we get is static; we are not able to guarantee that it will satisfy our performance constraint in Eq.~\eqref{eq:marginalcoverage}. 

Table~\ref{tab:test_bound} gives results for our models when using either the Meta or SM confidence measures (which we threshold with $\bm{\tau}$). We use half of $\mathcal{D}_{\mathrm{cal}}$ to find the minimal threshold that provides $\epsilon$-consistency. Then, we evaluate the threshold on the second half of $\mathcal{D}_{\mathrm{cal}}$ to get the empirical error. We compute the test set bound on this error with a confidence of $\delta = 10^{-2}$. As expected, the lower bound we compute is often significantly below $1 - \epsilon$, as it reflects the uncertainty that our measured consistency is accurate. Often the measured empirical consistency is also slightly below $1 - \epsilon$. At a high level, the overall consistency vs.\ efficiency trade-off is otherwise broadly similar to the one obtained by the Shared CP calibration.

\subsection{Nonconformity measure comparison} \label{app:nocomf_measures}
The test statistic used for a conformal prediction is typically called a nonconformity measure (i.e., in our work this is $\mathcal{M}_k(x)$). We experiment with different nonconformity measures as drop-in replacements for $\mathcal{M}_k(x)$, and report the results in Table~\ref{tab:ablation_noncomf}. The conformal calibration guarantees validity with any measure, even a random one, as long as they retain exchangeability. Good measures are ones that are statistically efficient, and will minimize the number of layers required for prediction at the required confidence level. This is a result of smaller $\cset$ sets, that tightly cover the inconsistent layers (and hence are more judicious with the complement, $\cset^c$). To be consistent with previous work where softmax metrics are used~\citep[such as][]{schwartz-etal-2020-right}, we use $p_k^{\mathrm{max}}$ as our non-Meta baseline in the main paper. In some settings, however, $p_k^{\mathrm{diff}}$ performs slightly better.

% %%%%%%

% \subsection{Test set bound method} \label{sec:testbound}
% As an alternative to conformal calibration, one can compute a binomial tail bound~\citep{langford2005} to obtain a lower bound on the consistency of $\mathcal{G}$. The two methods, though, differ both in their guarantees and in functionality. The test set bound is conditioned on the calibration data, whereas CP's guarantee is marginal over samples of the calibration set. Also, CP readily works in an online setting. In terms of usage, unlike CP, with the test set bound, the user cannot aim for a specific lower bound (e.g., $\epsilon$) but has to derive the bound post hoc. A workaround is to search for the desired bound \citep{geifman2017selective}.

% For completeness, we evaluate the test set bound on our models using the Meta and SM as confidence measures. We use half of $\mathcal{D}_{\mathrm{cal}}$ to find the minimal threshold that provides $\epsilon$-consistency. Then, we evaluate the threshold on the second half of $\mathcal{D}_{\mathrm{cal}}$ to get the empirical error. We compute the test set bound on this error with a confidence of $10^{-2}$.

% The results on the classification development sets are summarized in Table~\ref{tab:test_bound}. As expected, the achieved lower bound on the consistency is below $1 - \epsilon$ as it is computed after the threshold tuning. The overall consistency vs.\ efficiency trade-off is broadly similar to the one obtained by the Shared CP calibration.

% %%%%%%

\begin{figure*}[!t]
\small
\centering
\footnotesize
\begin{subfigure}{\textwidth}
\includegraphics[width=1\linewidth]{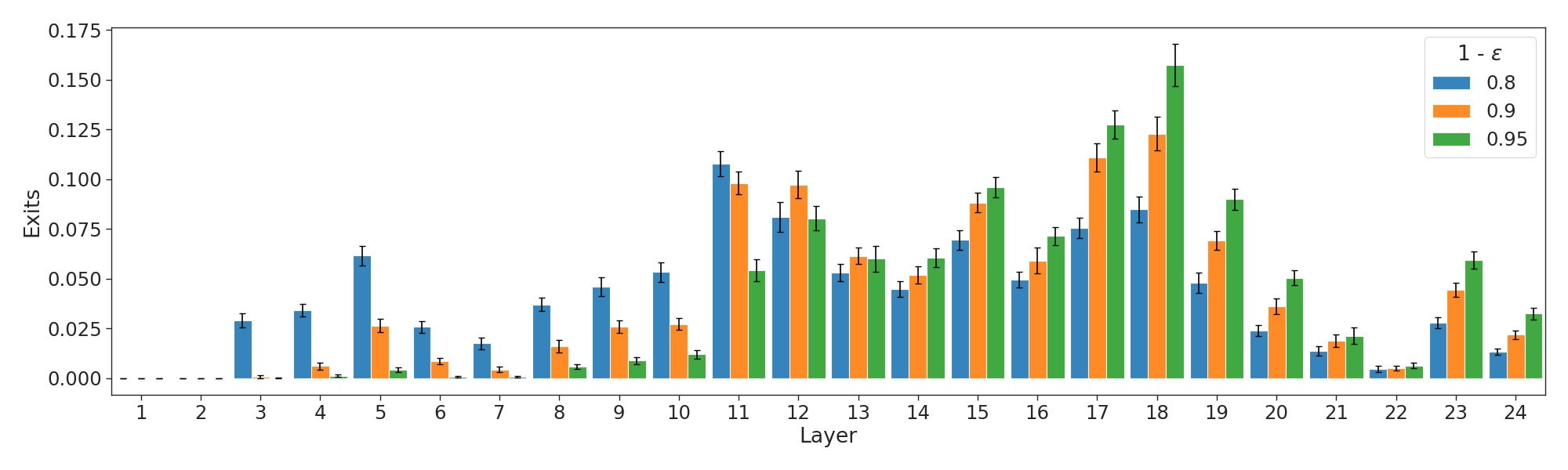} 
\caption{VitaminC}
\end{subfigure}

% \vspace*{-1.4\baselineskip}
\begin{subfigure}{\textwidth}
\includegraphics[width=1\linewidth]{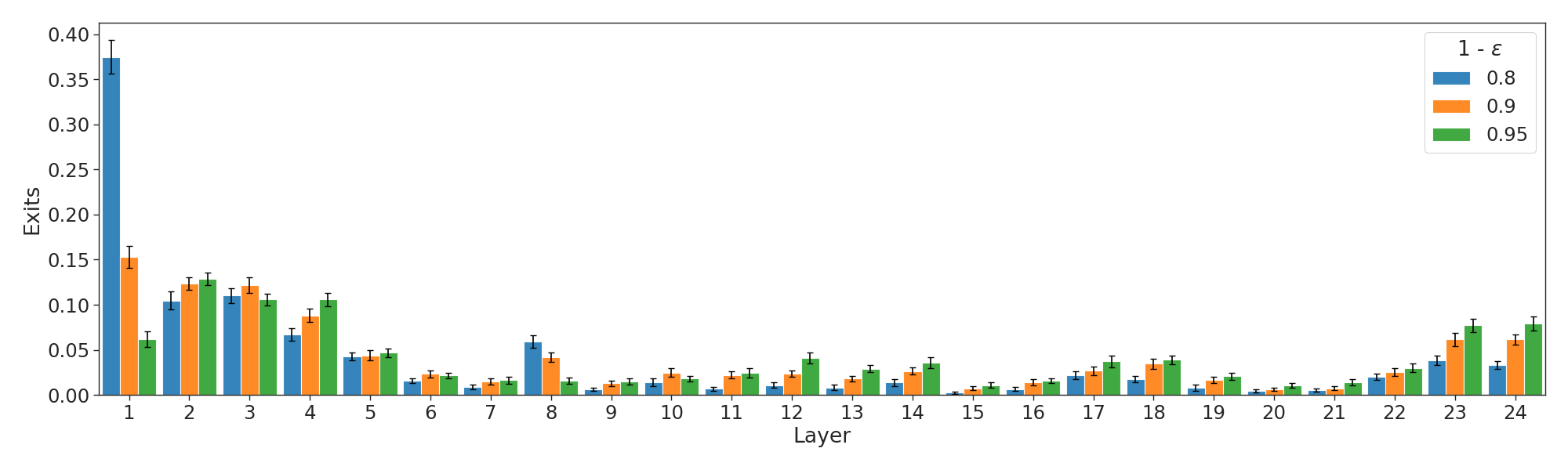}  
\caption{AG News}
% \vspace*{-1.8\baselineskip}
\end{subfigure}

\begin{subfigure}{\textwidth}
\includegraphics[width=1\linewidth]{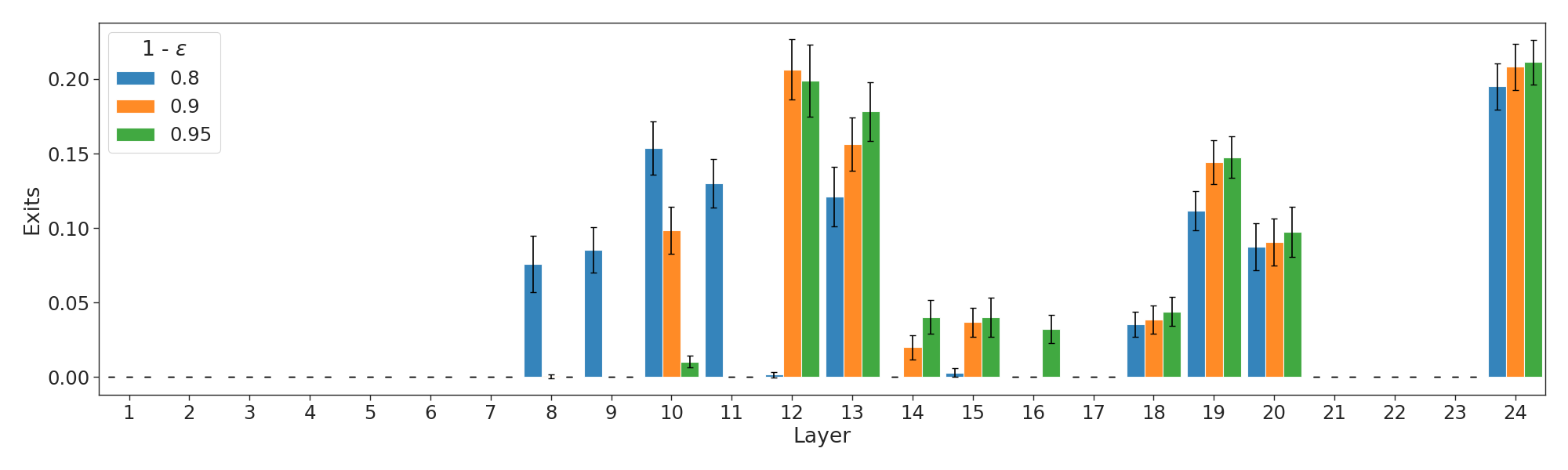}  
\caption{STS-B}
% \vspace*{-1.8\baselineskip}
\end{subfigure}

% \vspace*{-0.5\baselineskip}
\caption{Distribution of exit layers per tolerance level $\epsilon$ (dev sets) with our Shared/ Meta Albert-xlarge model.\\ See Figure~\ref{fig:imdb_hist} for IMDB.}
% \vspace{-10pt}
\label{fig:exit_dist_additional}
\end{figure*}

\subsection{Exit layer statistics}
Figure~\ref{fig:exit_dist_additional} depicts the distribution of exit layers for the different tasks with three reference tolerance levels. Reducing $\epsilon$ requires greater confidence before exiting, resulting in later exits on average. We provide example inputs with their respective exit layer in Appendix~\ref{app:example_outputs}.

\begin{table*}[!t]
\centering
\small
\vspace{-3pt}

\begin{tabular}{l|dddd|dddd}
\toprule
\multicolumn{1}{c|}{Method} &
  \multicolumn{4}{c|}{\textbf{Amortized time ($100\cdot T_{\mathcal{G}} / T_{\mathcal{F}}$})} &
  \multicolumn{4}{c}{\textbf{MACs reduction} ($|\mathcal{F}| / |\mathcal{G}|$)} \\
  \multicolumn{1}{c|}{} & 
  \multicolumn{1}{c}{IMDB} &
  \multicolumn{1}{c}{VitaminC} &
  \multicolumn{1}{c}{AG News} &
  \multicolumn{1}{c|}{STS-B} &
  \multicolumn{1}{c}{IMDB} &
  \multicolumn{1}{c}{VitaminC} &
  \multicolumn{1}{c}{AG News} &
  \multicolumn{1}{c}{STS-B} \\
\midrule
% \multicolumn{1}{l}{\textit{$\underline{1-\epsilon = 0.95}$}:} &  \\[1.2mm]
\rowcolor{Gray}
Thres./ SM & 
85.56   &	102.12  &  	112.52  &	\multicolumn{1}{c|}{\text{N/A}}       &   1.45    &	1.20    &	1.08    &	\multicolumn{1}{r}{\text{N/A}}   \\
\rowcolor{Gray}
Thres./ Meta & 
99 .85   &	109.93  &  	91.95   &	107.44   &	1.35    &	1.22    &	1.48    &	1.25    \\ 
Indep./ Meta &
89.25   &	109.57  &  	114.66  &	130.36  &	1.53    &	1.22    &	1.17    &	1.02    \\
Shared/ SM &
67.22   &	\textbf{90}.\textbf{41}   &  	69.99   &	\multicolumn{1}{c|}{\text{N/A}}      &	1.90    &	1.37    &	1.81    &	\multicolumn{1}{r}{\text{N/A}}       \\
Shared/ Meta &
\textbf{63} .\textbf{99} \ \%   &	94.97 \ \%   &  	\textbf{60}.\textbf{56} \ \%   &	\textbf{99}.\textbf{38} \ \%   &	\times \textbf{2}.\textbf{22}    &	\times \textbf{1}.\textbf{43}    &	\times \textbf{2}.\textbf{36}    &	\times \textbf{1}.\textbf{36}   \\

% \midrule
% \multicolumn{1}{l}{\textit{$\underline{1-\epsilon = 0.90}$}:} &  \\[1.2mm]
% \rowcolor{Gray}
% Thres./ SM & 
% 76.91   &	96.66   &	99.58   &	\multicolumn{1}{c|}{\text{N/A}}       &   1.63    &	1.27    &	1.23    &	\multicolumn{1}{r}{\text{N/A}}   \\
% \rowcolor{Gray}
% Thres./ Meta & 
% 87.22   &	103.59  &	77.87   &	104.01  &	1.57    &	1.30    &	1.78    &	1.30    \\
% Indep./ Meta &
% 84.88   &	103.85  &	99.44   &	113.00  &	1.62    &	1.30    &	1.36    &	1.18    \\
% Shared/ SM &
% 56.16   &	\textbf{84}.\textbf{86}   &	58.47   &	\multicolumn{1}{c|}{\text{N/A}}       &	2.33    &	1.46    &	2.22    &	\multicolumn{1}{r}{\text{N/A}}   \\
% Shared/ Meta &
% \textbf{54}.\textbf{53} \ \%   &	87.38  \ \%  &	\textbf{51}.\textbf{10} \ \%   &	\textbf{97}.\textbf{56} \ \%   &	\times \textbf{2}.\textbf{66}    &	\times \textbf{1}.\textbf{57}    &	\times \textbf{2}.\textbf{87}    &	\times \textbf{1}.\textbf{39}    \\

\bottomrule                               
\end{tabular}
\caption{Complementary results for Table~\ref{tab:speedup_results} with $1-\epsilon=0.95$.}

\vspace{-10pt}
\label{tab:speedup_results_app}
\end{table*}
\subsection{Albert-base results}\label{app:base_res}

\begin{figure*}[!t]
\small
\centering
\footnotesize
\begin{subfigure}{0.31\textwidth}
\includegraphics[width=1.05\linewidth]{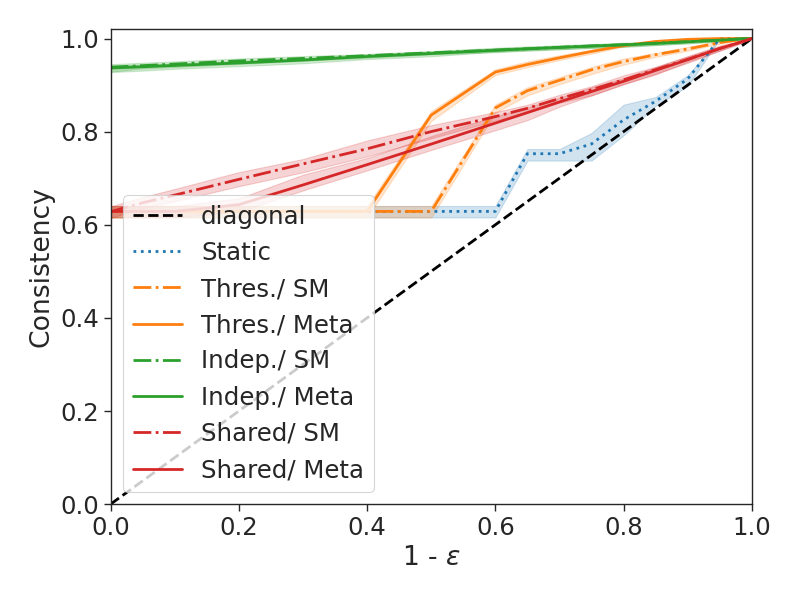} 
\end{subfigure}
%  \hfill
~
\begin{subfigure}{0.31\textwidth}
\includegraphics[width=1.05\linewidth]{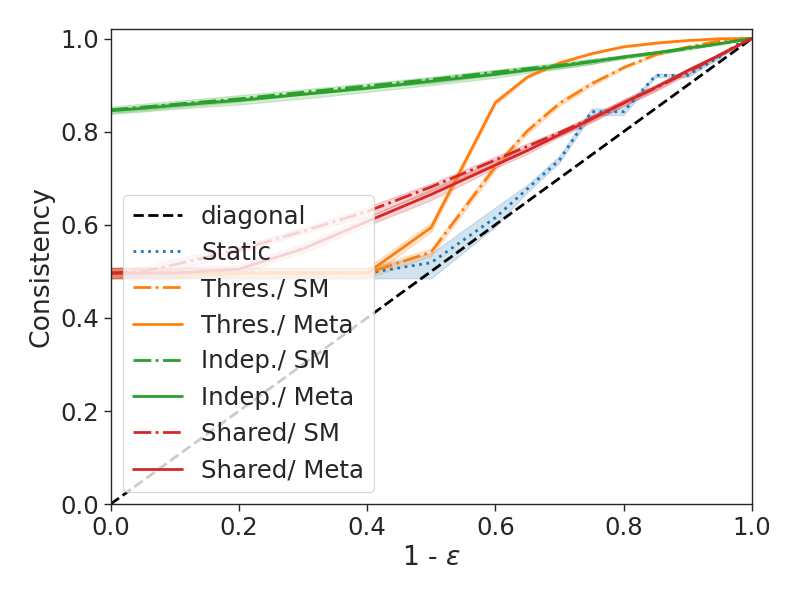}
\end{subfigure}
~
\begin{subfigure}{0.31\textwidth}
\includegraphics[width=1.05\linewidth]{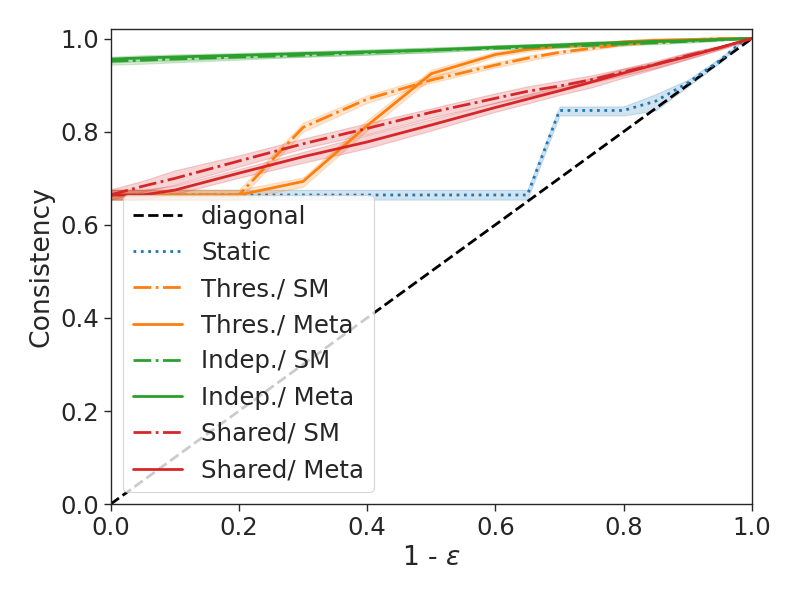} 
\end{subfigure}

\vspace*{-1.2\baselineskip}
\begin{subfigure}{0.31\textwidth}
\includegraphics[width=1.05\linewidth]{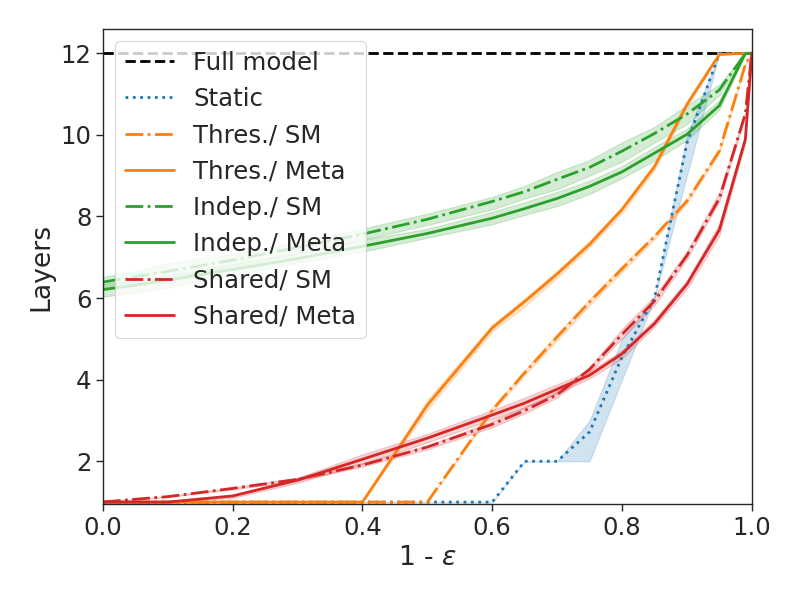} 
% \vspace*{-1.8\baselineskip}
\caption{IMDB}
\end{subfigure}
~
\begin{subfigure}{0.31\textwidth}
\includegraphics[width=1.05\linewidth]{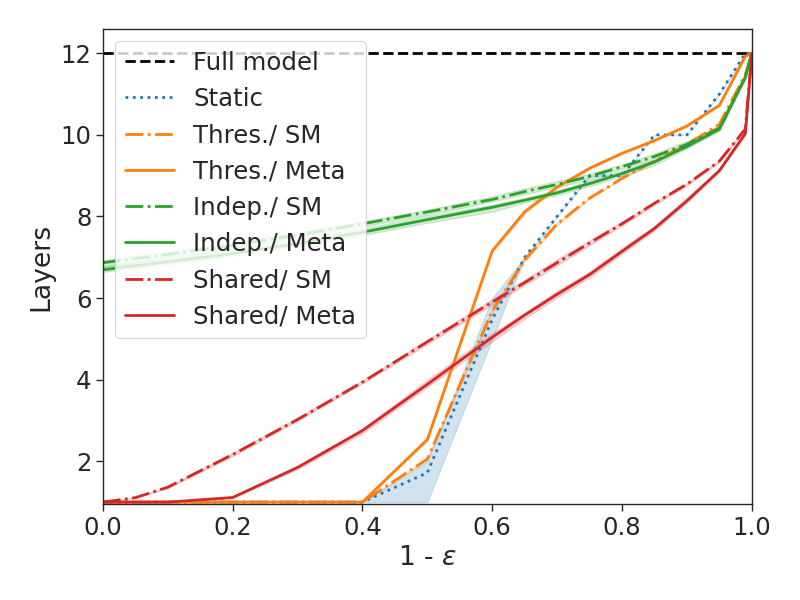}
% \vspace*{-1.8\baselineskip}
\caption{VitaminC}
\end{subfigure}
~
\begin{subfigure}{0.31\textwidth}
\includegraphics[width=1.05\linewidth]{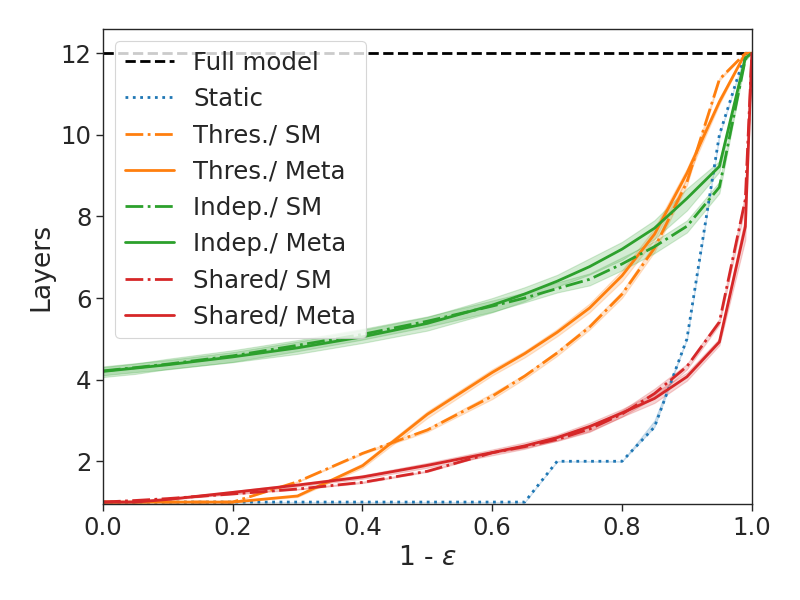} 
% \vspace*{-1.8\baselineskip}
\caption{AG News}
\end{subfigure}
% \vspace*{-0.5\baselineskip}

\begin{subfigure}{0.31\textwidth}
\includegraphics[width=1.05\linewidth]{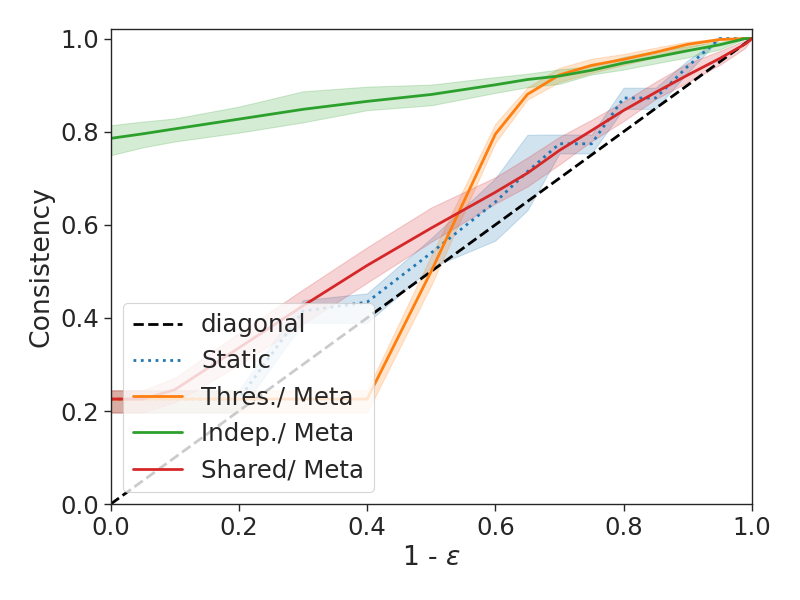} 
\hfill
\includegraphics[width=1.05\linewidth]{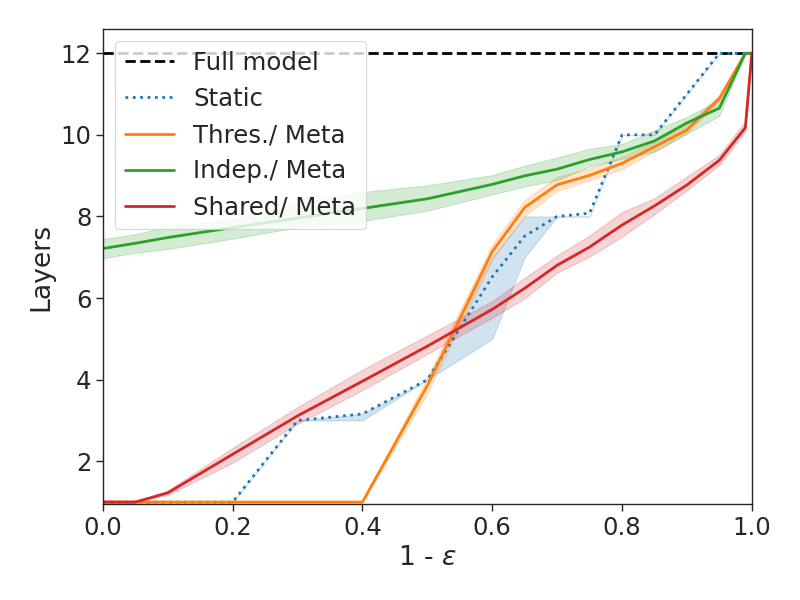} 
\caption{STS-B}
% \vspace*{-1.8\baselineskip}
\end{subfigure}

\caption{Development set results with an Albert-base 12-layers model as $\mathcal{F}$.}

% \vspace{-10pt}
\label{fig:cp_res_base}
\end{figure*}

% \begin{figure}[!t]
% \small
% \centering
% \footnotesize
% \begin{subfigure}{0.4\textwidth}
% \includegraphics[width=1.0\linewidth]{figures/base/stsb/const.png} 
% \end{subfigure}
% \vspace*{-1.8\baselineskip}
% \begin{subfigure}{0.4\textwidth}
% \includegraphics[width=1.0\linewidth]{figures/base/stsb/size.png} 
% % \vspace*{-1.8\baselineskip}
% \end{subfigure}
% % \vspace*{-0.5\baselineskip}
% \caption{Regression results for STSB evaluation set with an Albert-base 12-layers model.}
% % \vspace{-10pt}
% \label{fig:cp_reg_res_base}
% \end{figure}

Figure~\ref{fig:cp_res_base} reports the classification and regression results with an Albert-base 12-layers model. The trends are similar to the larger 24-layers version. Again, we see the efficacy of our Shared conformal calibration and the Meta nonconformity scores. For example, the AG News CAT Shared/ Meta model can preserve 95\% consistency while using less than 5 Transformer layers on average.

\subsection{RoBERTa-large results}\label{app:roberta_res}

\begin{figure*}[!t]
\small
\centering
\footnotesize
\begin{subfigure}{0.31\textwidth}
\includegraphics[width=1.05\linewidth]{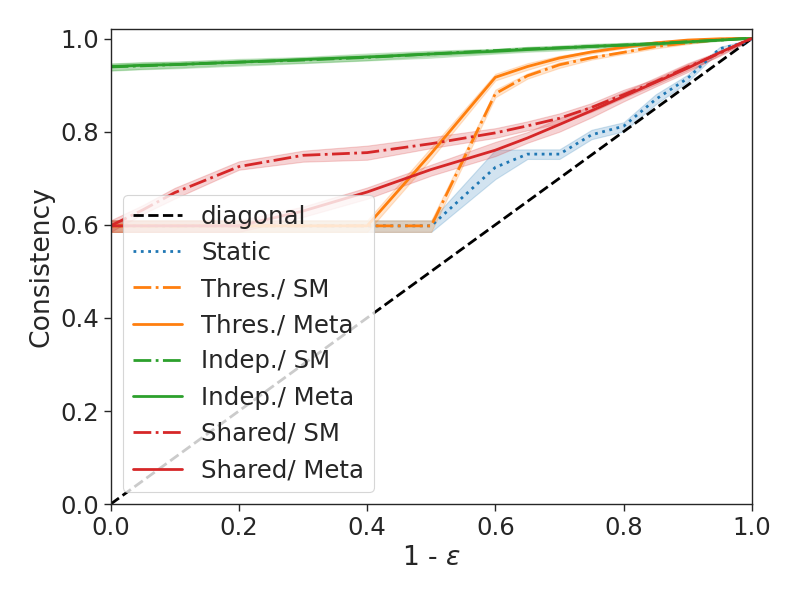} 
\end{subfigure}
%  \hfill
~
\begin{subfigure}{0.31\textwidth}
\includegraphics[width=1.05\linewidth]{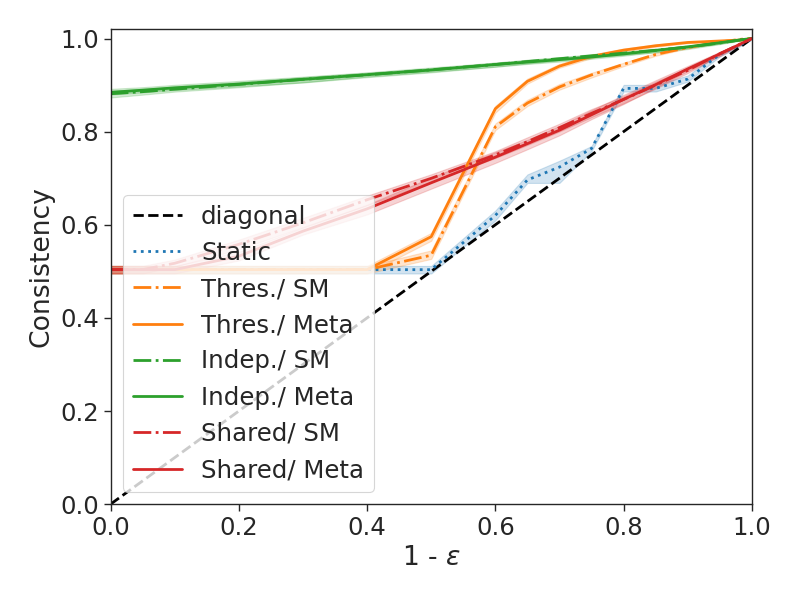}
\end{subfigure}
~
\begin{subfigure}{0.31\textwidth}
\includegraphics[width=1.05\linewidth]{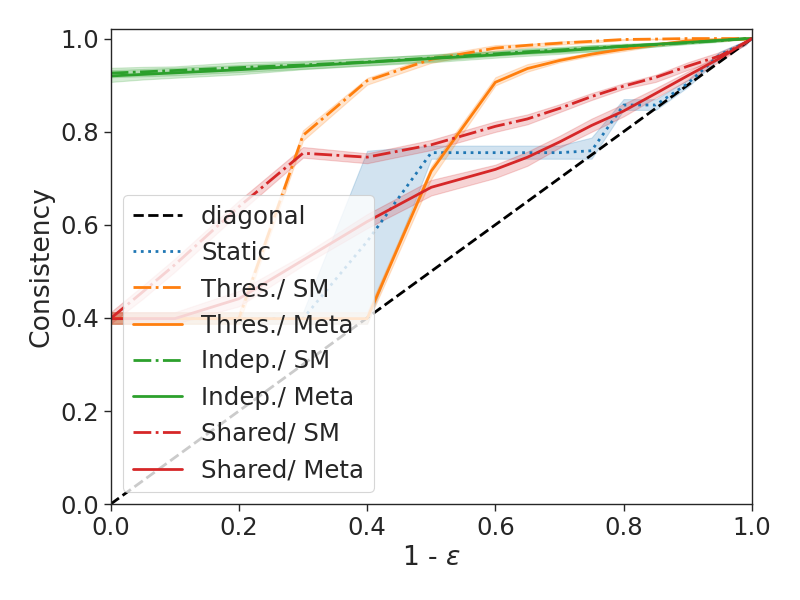} 
\end{subfigure}

\vspace*{-1.2\baselineskip}
\begin{subfigure}{0.31\textwidth}
\includegraphics[width=1.05\linewidth]{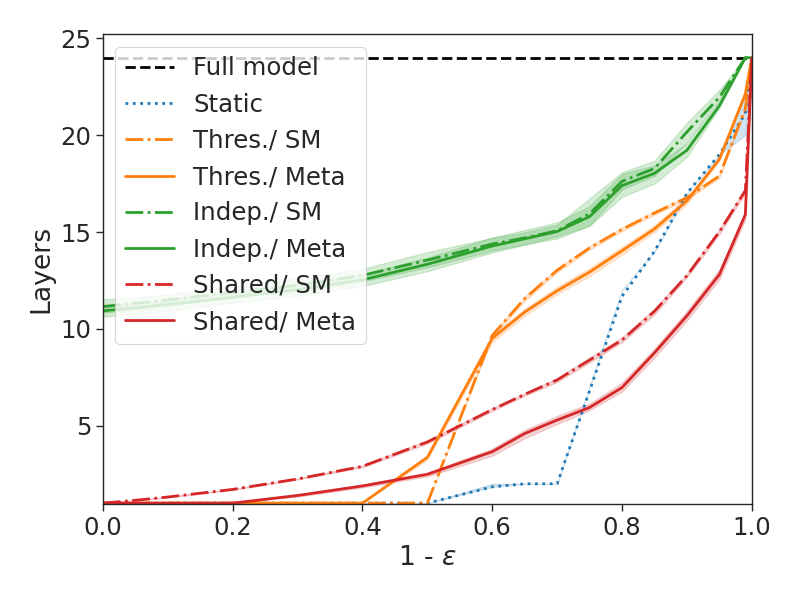} 
% \vspace*{-1.8\baselineskip}
\caption{IMDB}
\end{subfigure}
~
\begin{subfigure}{0.31\textwidth}
\includegraphics[width=1.05\linewidth]{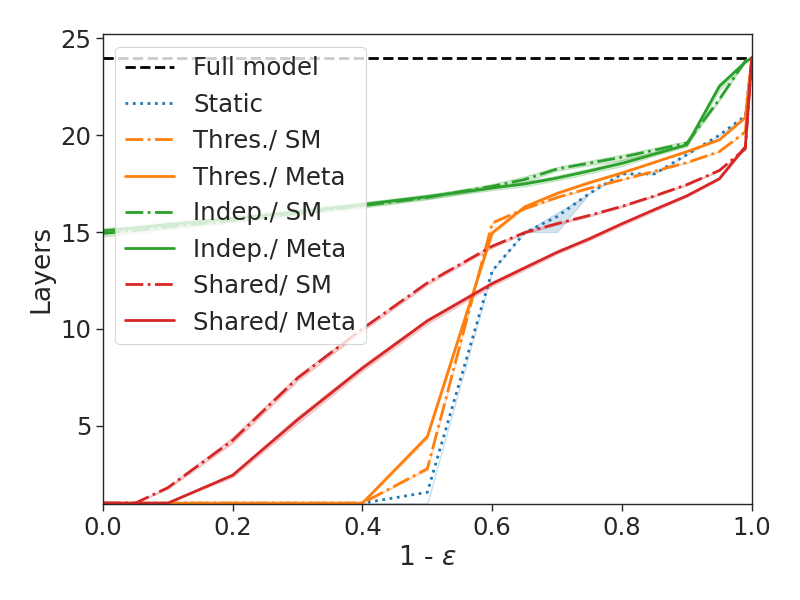}
% \vspace*{-1.8\baselineskip}
\caption{VitaminC}
\end{subfigure}
~
\begin{subfigure}{0.31\textwidth}
\includegraphics[width=1.05\linewidth]{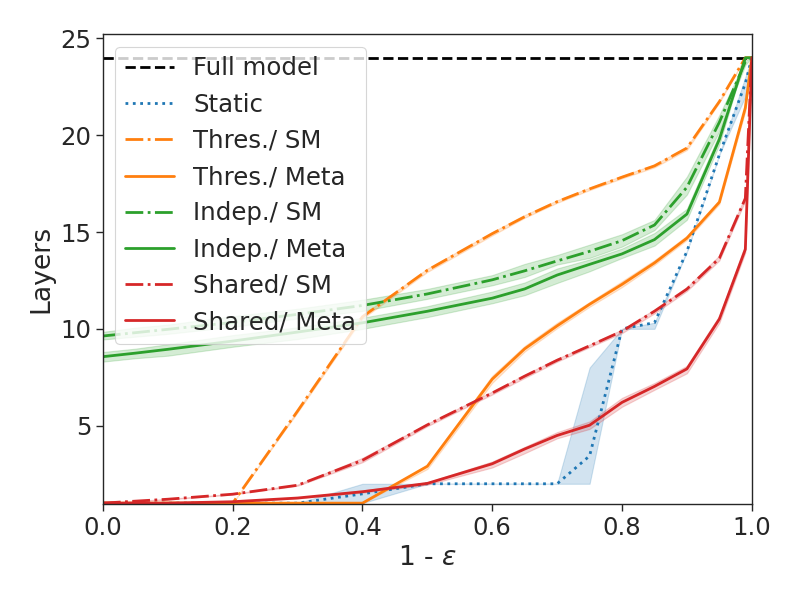} 
% \vspace*{-1.8\baselineskip}
\caption{AG News}
\end{subfigure}
% \vspace*{-0.5\baselineskip}

\begin{subfigure}{0.31\textwidth}
\includegraphics[width=1.05\linewidth]{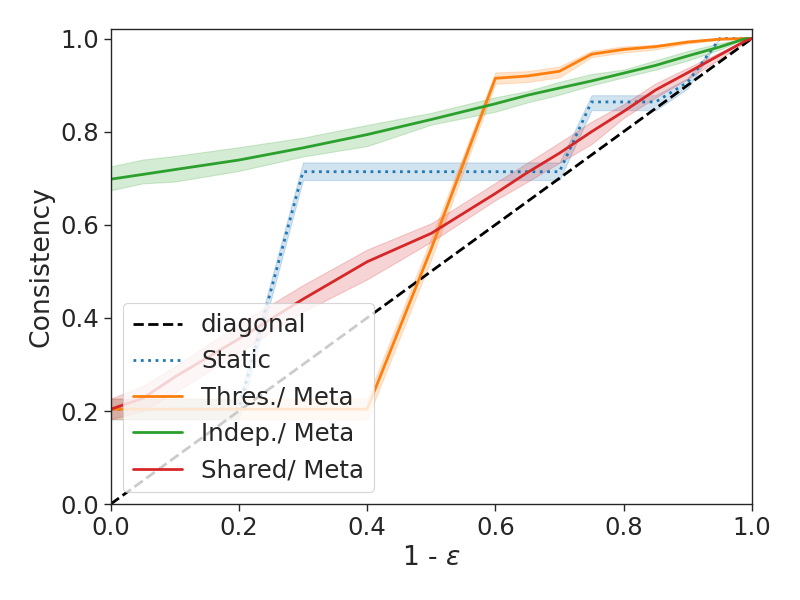} 
\hfill
\includegraphics[width=1.05\linewidth]{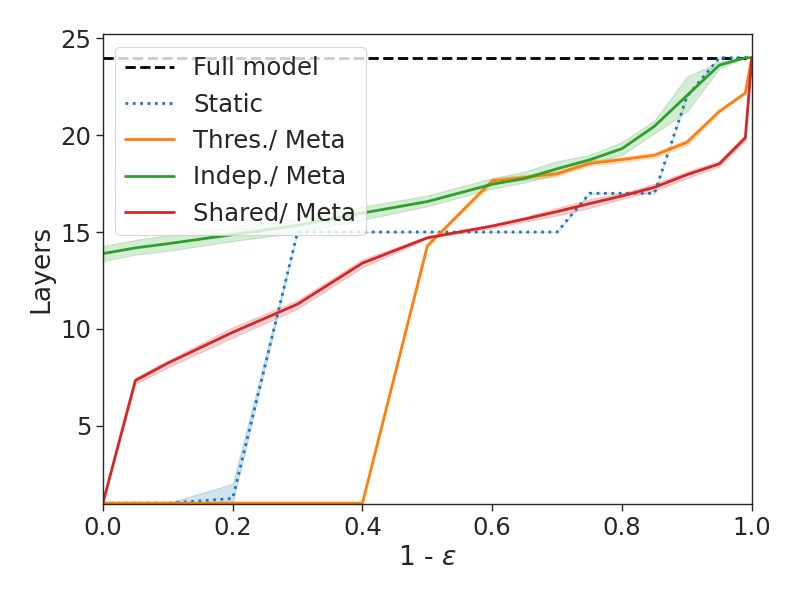} 
\caption{STS-B}
% \vspace*{-1.8\baselineskip}
\end{subfigure}

\caption{Development set results with an RoBERTa-large 24-layers model as $\mathcal{F}$.}

% \vspace{-10pt}
\label{fig:cp_res_roberta}
\end{figure*}

Figure~\ref{fig:cp_res_roberta} shows the results of our methods on top of the RoBERTa-large 24-layers Transformer. One main difference between RoBERTa and Albert, is that Albert shares the same parameters across all layers, essentially applying the same function recursively, whereas RoBERTa learns different parameters per layer. Yet, our method is agnostic to such differences and, as observed in the plots, results in similar trends. The value of our Meta classifier compared to the softmax response is even greater with the RoBERTa model.
\section{Example Predictions}\label{app:example_outputs}

Table~\ref{tab:examples} reports examples of inputs for different tasks and the number of layers that our Albert-xlarge CAT with $\epsilon=0.1$ required. These examples suggest that ``easier'' inputs (e.g., containing cue phrases or having large overlaps in sentence-pair tasks) might require less layers. In contrast, more complicated inputs (e.g., using less common language or requiring numerical analysis) can lead to additional computational effort until the desired confidence is obtained.

% We evaluate our methods on three classification tasks with varying label space size $|\mathcal{Y}|$, amount of training data, and difficulty: \textbf{IMDB}~\citep{maas-etal-2011-learning} sentiment analysis on movie reviews, \textbf{VitaminC}~\citep{Schuster2021} fact verification with Wikipedia articles, and \textbf{AG}~\citep{ag_news,zhang2015character} news topic classification. We also evaluate on the \textbf{STSB}~\citep{cer-etal-2017-semeval} semantic textual similarity regression task where $\mathcal{Y}\in[0,5]\subset \mathbb{R}$. Dataset statistics, along with the test set performance of our original $\mathcal{F}$ model (Albert-xlarge), are summarized in Table~\ref{tab:datasets}.
\begin{table*}[t]
\centering
\small
  \begin{tabular}{p{0.5cm}|p{1cm}|p{13.2cm}}
\toprule
Exit layer    &  Gold\newline label & Input   \\
\midrule
\multicolumn{1}{l}{} & \multicolumn{2}{l}{\textbf{IMDB} \citep{maas-etal-2011-learning}}  \\
% \midrule
\cmidrule{2-3} 
1 & Pos & Without question, film is a powerful medium, more so now than ever before, due to the accessibility of DVD/video, which gives the filmmaker the added assurance that his story or message is going to be seen by possibly millions of people. [...] \\
\cmidrule{2-3} 
4 & Neg & This movie was obscenely obvious and predictable. The scenes were poorly written and acted even worse.\\
\cmidrule{2-3} 
10 & Pos & I think Gerard's comments on the doc hit the nail on the head. Interesting film, but very long. [...]\\
\cmidrule{2-3} 
15 & Pos & here in Germany it was only shown on TV one time. today, as everything becomes mainstream, it's absolute impossible, to watch a film like this again on the screen. maybe it's the same in USA [...]\\
\cmidrule{2-3} 
20 & Neg & I tried to be patient and open-minded but found myself in a coma-like state. I wish I would have brought my duck and goose feather pillow... [...]\\
\cmidrule{2-3} 
24 & Neg &  Hypothetical situations abound, one-time director Harry Ralston gives us the ultimate post-apocalyptic glimpse with the world dead, left in the streets, in the stores, and throughout the landscape, sans in the middle of a forgotten desert. [...] \\

\midrule

\multicolumn{1}{l}{} & \multicolumn{2}{l}{\textbf{VitaminC} \citep{Schuster2021}}  \\
\cmidrule{2-3} 
3 & Sup &  \underline{Claim:} Another movie titled The SpongeBob Movie: Sponge on the Run is scheduled for release in 2020.\\
&  &  \underline{Evidence:} A second film titled The SpongeBob Movie : Sponge Out of Water was released in 2015, and another titled The SpongeBob Movie: Sponge on the Run is scheduled for release in 2020.\\
\cmidrule{2-3} 
5 & Sup &  \underline{Claim:} Julie Bishop offered a defence of her nation's intelligence cooperation with America.\\
&  &  \underline{Evidence:} The Australian Foreign Minister Julie Bishop stated that the acts of Edward Snowden were treachery and offered a staunch defence of her nation's intelligence co-operation with America.\\
\cmidrule{2-3} 
10 & NEI &  \underline{Claim:} The character Leslie hurts her head on the window in the film 10 Cloverfield Lane.\\
&  & \underline{Evidence:} Michelle realizes Howard was right and returns his keys.\\
\cmidrule{2-3} 
15 & Sup &  \underline{Claim:} Halakha laws are independent of being physically present in the Land of Israel.\\
&  & \underline{Evidence:} The codification efforts that culminated in the Shulchan Aruch divide the law into four sections, including only laws that do not depend on being physically present in the Land of Israel.\\
\cmidrule{2-3} 
20 & Sup &  \underline{Claim:} Germany has recorded less than 74,510 cases of coronavirus , including under 830 deaths.\\
&  & \underline{Evidence:} 74,508 cases have been reported with 821 deaths and approximately 16,100 recoveries.\\
\cmidrule{2-3} 
24 & NEI &  \underline{Claim:} For the 2015-16 school year , the undergraduate fee at USF is under \$43,000.\\
&  & \underline{Evidence:} Undergraduate tuition at USF is \$44,040 for the 2016-17 school year.\\

\midrule

\multicolumn{1}{l}{} & \multicolumn{2}{l}{\textbf{AG News} \citep{ag_news,zhang2015character}}  \\
\cmidrule{2-3} 
1 & Business & Crude Oil Rises on Speculation Cold Weather May Increase Demand Crude oil futures are headed for their biggest weekly gain in 21 months [...] \\
\cmidrule{2-3} 
5 & Sports & NHL Owner Is Criticized for Talking of Replacement Players The day before the regular season was supposed to open [...] \\
% \cmidrule{2-3} 
% 10 & World & North Korea Says the Tyrant is Bush, not Kim North Korea says it sees no reason to join a working-level meeting with the United States [...] \\
\cmidrule{2-3} 
15 & World & Scotch Whisky eyes Asian and Eastern European markets (AFP) AFP - A favourite tipple among connoisseurs the world over, whisky is treated with almost religious reverence on the Hebridean [...] \\
\cmidrule{2-3} 
20 & Business &  Arthritis drug withdrawn after trial A prescription painkiller used by more than 250,000 Australians to treat arthritis has been withdrawn from sale after a clinical trial found it doubled the risk [...] \\
\cmidrule{2-3} 
24 & Sci/Tech &  Airbus drops out of Microsoft appeal Aircraft builder withdraws its request to intervene in Microsoft's antitrust appeal; Boeing also forgoes intervention. \\

\midrule

\multicolumn{1}{l}{} & \multicolumn{2}{l}{\textbf{STS-B} \citep{cer-etal-2017-semeval}}  \\
\cmidrule{2-3} 
10 & 0.6 &  \underline{Sent. 1:} A child wearing blue and white shorts is jumping in the surf.\\
&  &  \underline{Sent. 2:} A girl wearing green twists something in her hands.\\
\cmidrule{2-3} 
15 & 2.8 &  \underline{Sent. 1:} Saudi Arabia gets a seat at the UN Security Council\\
&  &  \underline{Sent. 2:} Saudi Arabia rejects seat on UN Security Council\\
\cmidrule{2-3} 
20 & 4.2 &  \underline{Sent. 1:} a small bird sitting on a branch in winter.\\
&  &  \underline{Sent. 2:} A small bird perched on an icy branch.\\
\cmidrule{2-3} 
24 & 3.0 &  \underline{Sent. 1:} It depends entirely on your company and your contract.\\
&  &  \underline{Sent. 2:} It depends on your company.\\

\bottomrule
\end{tabular}
% \vspace*{-10pt}
\caption{Number of Transformer layers used for example inputs from the task's test sets with our Shared/Meta CAT with a tolerance level of $\epsilon=0.1$} \label{tab:examples}
\end{table*}

% \section{Supplemental Material}
% \label{sec:supplemental}

\end{document}